\RequirePackage{iftex}
\ifPDFTeX\else
  \PackageWarningNoLine{HACE}{Select pdfLaTeX as the compiler. The Meta font setup requires pdfLaTeX}%
\fi
\RequirePDFTeX

\RequirePackage{fix-cm}
\PassOptionsToPackage{table}{xcolor}
\documentclass{fairmeta}

\usepackage{lmodern}

\usepackage{amsmath,amsfonts,bm}

\def\eqref#1{equation~\ref{#1}}

\def\1{\bm{1}}

\DeclareMathAlphabet{\mathsfit}{\encodingdefault}{\sfdefault}{m}{sl}
\SetMathAlphabet{\mathsfit}{bold}{\encodingdefault}{\sfdefault}{bx}{n}

\usepackage{adjustbox,array,tabularx,threeparttable}
\usepackage{pifont,flafter,needspace,xurl,enumitem}
\usepackage{amssymb,amsthm,longtable}
\renewcommand{\sfdefault}{optimistic}

\hypersetup{
  pdftitle={Heterogeneous Agentic Coevolution},
  pdfauthor={Wenqian Ye, Ziwei Guan, Zi Yang, Shivani Modi, Cici Cong Wu,
    Yunlong He, Yuxin Chen, Zhengyu Zhang, Liang Luo, Venkatesh Ranganathan,
    Jianhui Sun, Yantao Yao, Wenlin Chen, Buyun Zhang, Ellie Dingqiao Wen,
    Henry Kautz, Aidong Zhang},
  pdfsubject={Heterogeneous Agentic Coevolution},
  colorlinks=true,
  linkcolor=metablue,
  citecolor=metablue,
  urlcolor=metablue
}

\renewcommand{\authorformat}[2][]{\mbox{{\sffamily\bfseries #2}\textsuperscript{#1}}}
\renewcommand{\affiliationformat}[2][]{\mbox{\textsuperscript{#1}#2}}

\definecolor{metabg}{HTML}{EDF2F8}
\definecolor{haceaccent}{HTML}{273F5B}
\definecolor{metablue}{HTML}{365F91}
\AtBeginDocument{\pretocmd{\titlelist}{\color{haceaccent}}{}{} }
\titleformat*{\section}{\Large\sffamily\bfseries\color{haceaccent}}
\titleformat*{\subsection}{\large\sffamily\bfseries\color{haceaccent}}
\titleformat*{\subsubsection}{\normalsize\sffamily\bfseries\color{haceaccent}}

\newcommand{\result}[2]{\ensuremath{#1_{\color{black!60}\pm #2}}}
\newcommand{\bestresult}[2]{\ensuremath{\mathbf{#1}_{\color{black!60}\pm #2}}}

\newcommand{\hacetablesetup}{%
  \fontsize{9}{11}\selectfont
  \setlength{\tabcolsep}{2.5pt}%
  \renewcommand{\arraystretch}{1.08}%
}

\definecolor{todocolor}{HTML}{7656A5}
\newif\ifshowtodos
\showtodostrue

\newtheorem{hacetheorem}{Theorem}
\newtheorem{haceproposition}{Proposition}

\theoremstyle{remark}

\theoremstyle{plain}

\title{%
  \raisebox{-2pt}[0pt][0pt]{\includegraphics[width=32pt,trim={68bp 279bp 59bp 355bp},clip]{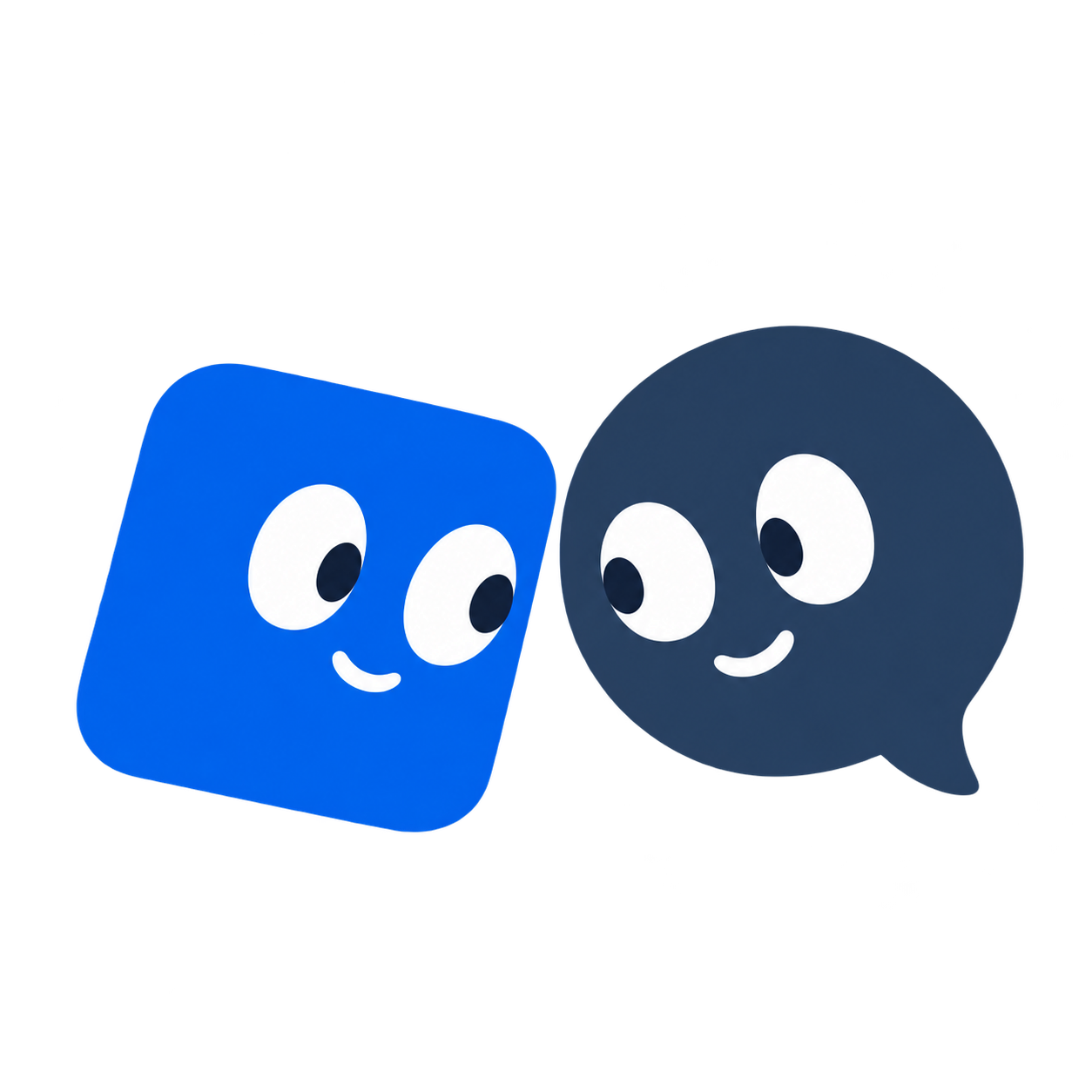}}\hspace{6pt}%
  \textcolor[HTML]{0064E0}{H}eterogeneous \textcolor[HTML]{0064E0}{A}gentic \textcolor[HTML]{0064E0}{C}o\textcolor[HTML]{0064E0}{e}volution}

\author[1,2,*]{Wenqian Ye}
\author[1]{Ziwei Guan}
\author[1]{Zi Yang}
\author[1]{Shivani Modi}
\author[1]{Cici Cong Wu}
\author[1]{Yunlong He}
\author[1]{Yuxin Chen}
\author[1]{Zhengyu Zhang}
\author[1]{Liang Luo}
\author[1]{Venkatesh Ranganathan}
\author[1]{Jianhui Sun}
\author[1]{Yantao Yao}
\author[1]{Wenlin Chen}
\author[1,\textdagger]{Buyun Zhang}
\author[1,\textdagger]{Ellie Dingqiao Wen}
\author[2,\textdagger]{Henry Kautz}
\author[2,\textdagger]{Aidong Zhang}

\affiliation[1]{Meta AI}
\affiliation[2]{University of Virginia}
\contribution[*]{Work done at Meta}

\correspondence{\mbox{Wenqian Ye (\email{wenqian@virginia.edu})}, \mbox{Buyun Zhang (\email{buyunz@meta.com})}}
\abstract{Recent advances in self-improving LLM agents have demonstrated the potential of AI systems that drive their own intelligence. However, many domain-specific applications rely on \textit{specialized foundation models} that possess strong domain expertise but remain dependent on human-designed improvements. Extending recursive self-improvement to these models presents a key challenge: as a specialist improves, the strategies used to advance it must also adapt. To address this challenge, we propose \textbf{Heterogeneous Agentic CoEvolution (HACE)}, a framework that jointly evolves specialized foundation models and LLM agents through reciprocal learning. HACE improves the specialist through supervised learning while training the agent with on-policy reinforced self-training, guided by verified performance gains on the current specialist. The agent's exploration, in turn, informs subsequent specialist training, enabling both models to adapt to each other's progress. To sustain this reciprocal learning, HACE further employs evolutionary search to optimize their interaction harness, adapting how the models exchange information and select training experiences as their capabilities evolve. We evaluate HACE across three challenging domains: recommendation systems, protein function prediction in biology, and crystal stability prediction in material discovery. We show that HACE outperforms the evaluated self-improving methods on the primary task metrics across all three domains. These results shed light on heterogeneous coevolution as a universal approach to improving specialized foundation models through reciprocal interaction with a learning agent.}

\begin{document}
\maketitle
\begingroup
\renewcommand{\thefootnote}{\fnsymbol{footnote}}
\footnotetext[2]{These authors jointly supervised this work}
\endgroup

\begingroup
\setlength{\intextsep}{6pt plus 2pt minus 1pt}
\begin{figure}[!ht]

    \centering
    \includegraphics[width=.97\linewidth]{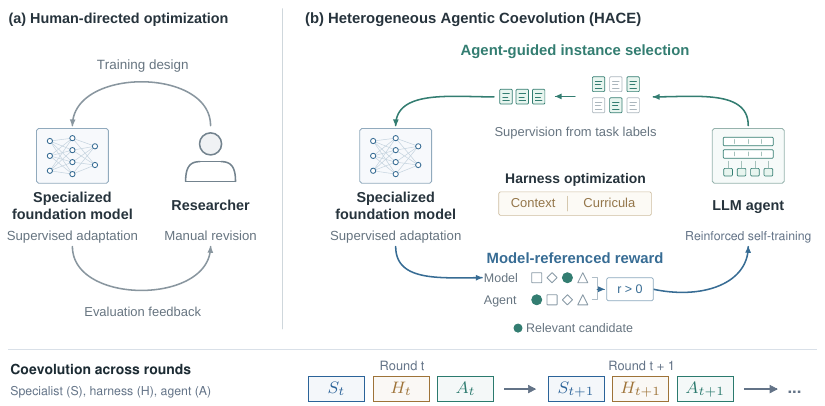}
    \caption{From human-directed optimization to heterogeneous agentic coevolution. HACE brings together \textit{specialized foundation models} and \textit{LLM agents} into a shared learning process, where each helps determine how the other improves. Across rounds, both learners and their interaction harness evolve, creating new opportunities for reciprocal improvement.}
    \label{fig:teaser}

\end{figure}
\endgroup
\clearpage

\section{Introduction}
Lately, large language model (LLM) agents are beginning to automate the experimental work behind model development. Agents can revise their implementations, improve their interaction harnesses, and modify the procedures that generate further improvements \citep{zhang2026dgm,lee2026metaharness,zhang2026hyperagents}. They can also continue training on experience collected during adaptation and exploration \citep{zweiger2025seal,yuksekgonul2026discover}. These advances suggest that model development itself can become an autonomous learning process. Extending this progress to scientific AI requires bringing the domain expertise of specialized foundation models into this process.

In scientific AI, specialized foundation models generally learn reusable domain knowledge through large-scale pretraining. Protein language models learn biological regularities from sequences \citep{lin2023esm2}, while pretrained interatomic potential models learn the energies and forces governing material behavior \citep{deng2023chgnet}. In recommendation, architectures such as HSTU \citep{zhai2024hstu} and Wukong \citep{zhang2024wukong} extend this scaling approach to large collections of user interactions. These models acquire domain expertise that often \textit{exceeds} that of general LLM agents due to specific designs. Yet they still cannot self-evolve like LLM agents. 
Once supervised training converges and further training yields no improvement, researchers have to manually identify useful designs to improve the specialized foundation model, such as better feature engineering or algorithmic changes, as shown in Figure~\ref{fig:teaser}. 

Recent developments in agentic systems pave a feasible path to automate this process. Their ability to interpret feedback, explore alternatives, and revise executable procedures allows them to influence how another model learns. For example, Eureka  \citep{ma2024eureka} demonstrates this possibility by using LLMs to design reward programs for training robot policies autonomously. However, most self-evolving systems still rely on proprietary LLMs whose parameters remain fixed while agent code and interaction harnesses evolve \citep{zhang2026dgm,lee2026metaharness,zhang2026hyperagents}. The experience collected during improvement thus does not update the underlying language model. For specialized foundation models, the challenge is for the LLM agent model to keep improving the specialized model as the specialized model advances. An intervention that improves the initial model can become redundant after further training. An agent that learns to address the initial model's weaknesses, therefore, faces a changing task as those weaknesses are resolved. Continued progress requires the agent to learn continually with the model when it improves. We refer to this mechanism as \textit{reciprocal learning}.


To close the gap, we propose a novel framework,  \textbf{Heterogeneous Agentic Coevolution (HACE)}, which coevolves a specialized foundation model (i.e., the specialist) with a heterogeneous LLM agent which learns to guide the coevolution. The key idea is to improve the specialist through supervision while also training the agent to keep advancing it under an adaptive curriculum. Agent learning is guided by verified gains over the current specialist. The agent's exploration, in turn, guides the selection of experiences for subsequent specialist training. As the specialist improves, it changes which agent contributions are useful. As the agent improves, it changes the learning opportunities available to the specialist. This reciprocal feedback loop seamlessly connects supervised specialist learning with on-policy reinforced self-training. 

We need to ensure that these models interact and develop their capabilities in a continuously improving direction. What the agent sees and which experiences reach specialist training determine the opportunities that either model can use. Thus, we also propose an evolutionary interaction harness in HACE that governs information exchange and training experience selection. Evolutionary search evaluates harness changes through their effects on task performance and specialist learning. A better harness can expose opportunities that the current models have not utilized, while model updates change which harness designs remain effective. HACE thus improves the specialist while learning how to organize its continued improvement. Domain expertise, agentic capabilities, and the process connecting their learning coevolve together.

To show the effectiveness, we comprehensively evaluate HACE in three challenging domains, including recommendation systems, protein function prediction in biology, and crystal stability prediction in material discovery. In each domain, we begin with specialists already \textit{trained to convergence}, where continued training under the original supervised procedure yields no further improvement. This setting stress-tests whether interaction with an evolving agent can create opportunities for learning beyond that plateau. Across diverse specialist architectures and large-scale datasets, HACE achieves state-of-the-art performance and outperforms all evaluated evolution methods. These results demonstrate that HACE enables specialists to continue improving beyond supervised convergence. HACE brings their existing domain expertise into autonomous self-improvement while developing the agentic capabilities needed to advance it. 

\textbf{To summarize, our main contributions are as follows:}

\begin{itemize}[leftmargin=1.5em, labelsep=0.5em]
    \item We introduce HACE, an evolutionary framework for autonomous self-improvement that brings specialized foundation models and LLM agents into a shared coevolution process.
    \item We propose a novel reciprocal learning mechanism in which verified agent contributions guide agent learning, agent exploration also informs specialist training, and an evolving harness adaptively coordinates their reciprocal interaction.
    \item We empirically demonstrate the effectiveness of HACE across three challenging domains: recommendation, protein function prediction, and material discovery. The results show that HACE can substantially push the domain-specific performance beyond the capability boundary.
\end{itemize}
\section{Related Work}

\paragraph{Agent and harness optimization.}
Recent work uses execution feedback to improve agent programs, interaction harnesses, and the procedures that propose further improvements \citep{hu2025adas,zhang2026dgm,lee2026metaharness,luo2026autodesign,zhang2026hyperagents,cemri2026adaevolve,liu2026evox}. Other methods improve the underlying model through parameter updates. Reinforced self-training learns from responses selected by task rewards \citep{singh2024beyond}, while self-adaptation and test-time learning use generated experience or task feedback for further training \citep{zweiger2025seal,yuksekgonul2026discover}. HACE connects harness search with parameter learning in both a specialized foundation model and an LLM agent. The harness controls their interaction and the experience used to train each model.

\paragraph{Coevolution and reciprocal learning.}
Coevolution can be traced to reciprocal evolutionary responses between interacting populations in biology \citep{janzen1980coevolution}. In machine learning, many methods have connected the development of reward programs, environments, or optimization procedures with the learners they guide \citep{ma2024eureka,liu2026escher,wang2019poet,dennis2020emergent,chen2026coharness,yang-etal-2026-coevolve,li2026heterogeneous}. HACE studies this dependence between models with different capabilities and learning signals. Agent exploration can change which instances train the specialized model, whose predictions in turn determine which agent responses pass verification. Our analysis characterizes this dependence through the distribution of verified experience and identifies when an update to either learner changes its allocation across instances.

\paragraph{Agent distillation.}
Agent distillation transfers reasoning and tool use from LLM agents to smaller language models \citep{liu2026structured,kang2025agentdistillation}. MAD-OPD derives supervision from teachers that debate over the student's current trajectories \citep{wang2026madopd}. These methods update the student while keeping teacher parameters fixed. HACE is \textit{distinct} from this line of work and takes a step forward, where an LLM agent with lower overall domain performance can \textit{still} improve on the specialized foundation model. Agent exploration guides instance selection for specialist training, while verified contributions guide response selection back for agent learning. Both models can therefore coevolve reciprocally without assigning either a fixed teacher role.

\section{Heterogeneous Agentic Coevolution}
\label{sec:hace_method}

In this section, we formally present Heterogeneous Agentic Coevolution (HACE), a framework for reciprocal self-improvement between specialized foundation models and LLM agents. Figure~\ref{fig:overview} illustrates how an evolving harness connects specialist adaptation, joint selection, and agent learning within one coevolution loop. In Section~\ref{sec:hace_setup}, we describe how specialist predictions can guide agent exploration and how the resulting experience selects instances for specialist training with original task labels. In Section~\ref{sec:hace_selection}, we incorporate the proposed harness together with the adapted specialist, then use the retained pair to identify fresh responses for on-policy agent learning. These updates allow each learner to shape the training experience of the other precisely. In Section~\ref{sec:hace_theory}, we establish a source risk bound that relates specialist learning to the coverage of accepted experience within the coevolution.

\begin{figure}[!htb]
    \centering
    \includegraphics[width=\linewidth]{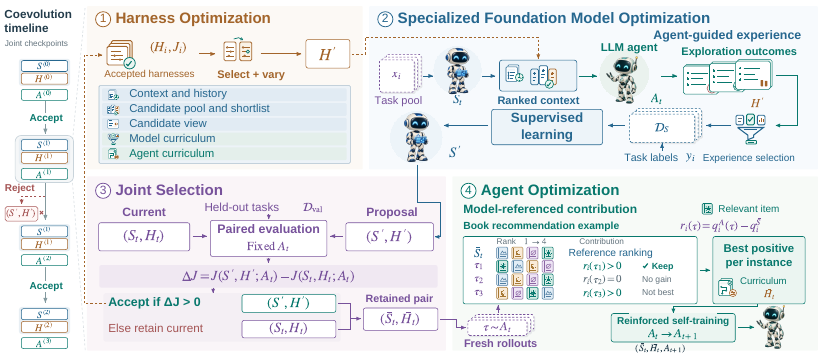}
    \caption{Overview of HACE. Agent exploration selects specialist training experience. Joint validation retains the specialist and harness, after which fresh verified responses train the agent and shape later exploration. The on-policy agent optimization step uses reinforced self-training.}
    \label{fig:overview}
\end{figure}

\subsection{Reciprocal Learning}
\label{sec:hace_setup}
\label{sec:hace_harness}
\label{sec:hace_model}

Let $S$ denote the specialist, $A$ the agent, and $H$ the harness. Round $t=0,1,\ldots$ starts from $(S_t,H_t,A_t)$. We propose a harness $H'$ by modifying a parent from an archive of accepted harnesses. Its fields control the task information and specialist context shown to the agent, how responses are interpreted, and which examples train each learner. Let $\mathcal I$ index the finite nonempty training pool of labeled instances $(x_i,y_i)$.

Under $H'$, the agent receives a prompt constructed from a training input $x_i$ and the predictions of $S_t$, without access to its label $y_i$. We evaluate each response $\tau$ against $y_i$ using the same task constraints as the specialist. Let $q_i^A(\tau)$ and $q_i^S$ denote their respective task qualities, with larger values preferred. The response contribution is
\begin{equation}
r_i(\tau)=q_i^A(\tau)-q_i^S.
\label{eq:hace_contribution}
\end{equation}
A response passes verification when its prediction is valid, supervision is evaluable, and $r_i(\tau)>0$. An agent with lower average task quality can still provide verified responses on individual instances. The specialist and harness used in each collection determine both quality terms. Appendix~\ref{app:hace_tasks} specifies the task quality and verification rules.

The harness also determines which prediction candidates the agent can access. Additional sampling cannot recover a relevant candidate excluded from that view. Changing the view can therefore create opportunities that further sampling under the same harness would miss, while the curriculum determines which exploration outcomes receive the specialist's training budget.

The harness uses these outcomes to select a nonempty set of training indices $\mathcal I_t^S\subseteq\mathcal I$. Its curriculum can prioritize verified improvements, unsuccessful exploration, or uncertain joint predictions, with uniform sampling as an option. Write $S_\phi$ for the specialist with trainable parameters $\phi$, which may belong to the foundation model or its prediction head. A fixed optimization budget starting from $S_t$ produces $S'$ using
\begin{equation}
\mathcal L_{S,t}(\phi)=\frac{1}{|\mathcal I_t^S|}\sum_{i\in\mathcal I_t^S}
\ell_{\mathrm{task}}(S_\phi(x_i),y_i),
\label{eq:hace_model_loss}
\end{equation}
where $\ell_{\mathrm{task}}$ is the original supervised task loss. Agent exploration determines where to train, while task labels determine what the specialist learns. The harness thus affects both current predictions and the specialist produced by adaptation. We evaluate $H'$ and $S'$ \textbf{together after specialist adaptation} to account for both effects.

\subsection{Joint Selection}
\label{sec:hace_selection}
\label{sec:hace_agent}

Joint selection tests whether specialist adaptation under the harness improves joint predictions. For a candidate $c$ in the view for instance $i$, let $s_i(c)$ be its specialist score after any stored calibration, with larger values preferred, and $a_i(c)$ its agent score. Joint prediction ranks candidates by
\begin{equation}
s_i(c)+\lambda a_i(c),\qquad
\lambda\geq0.
\label{eq:hace_fusion}
\end{equation}
Here $a_i(c)=1/\log_2(u+1)$ when $c$ appears at position $u$ in the agent ranking, and zero otherwise. The weight $\lambda$ controls the agent's influence, with zero recovering the specialist ordering within the same view.

Training, validation, and final evaluation use disjoint splits. With $A_t$ fixed, we compare the proposed and current pairs on the same validation batch $\mathcal B_t$. Let $J_{\mathcal B_t}(S,H,A)$ be the mean joint task quality for state $(S,H,A)$. Both states use the same two folds, fitting their own weights from a fixed grid on one fold and evaluating on the other, then exchanging roles. The proposal is accepted when
\begin{equation}
\Delta J_t=J_{\mathcal B_t}(S',H',A_t)-J_{\mathcal B_t}(S_t,H_t,A_t)>0.
\label{eq:hace_acceptance}
\end{equation}
Acceptance retains $(S',H')$ and records the harness and score. Otherwise, we restore $(S_t,H_t)$ and the specialist training state. Write $(\bar S_t,\bar H_t)$ for the pair retained after this decision.

The retained pair now defines the context and reference for agent learning. A specialist update can change which responses improve on its predictions, and a harness update can change the information available to the agent. We therefore collect fresh responses with $A_t$ under $(\bar S_t,\bar H_t)$, keep the original response with the largest positive contribution for each eligible instance, and apply the agent curriculum. This makes agent supervision depend on the current specialist's remaining opportunities for improvement.

Let $\mathcal Z_t$ be the indexed dataset of selected prompt and response concatenations, with repeated records counted separately. For a record $z$, let $|z|$ count its nonpadding tokens, $z_\ell$ be its token at position $\ell$, and $z_{<\ell}$ its preceding prefix. Write $p_\theta(z_\ell\mid z_{<\ell})$ for the agent's next-token probability before decoding transformations, with trainable parameters $\theta$. For nonempty $\mathcal Z_t$, on-policy reinforced self-training~\citep{singh2024beyond} minimizes
\begin{equation}
\mathcal L_{A,t}(\theta)=-\frac{1}{|\mathcal Z_t|}\sum_{z\in\mathcal Z_t}
\frac{1}{|z|-1}\sum_{\ell=2}^{|z|}\log p_\theta(z_\ell\mid z_{<\ell}).
\label{eq:hace_agent_loss}
\end{equation}
Each record has at least two tokens. The loss averages predicted tokens within each record, including prompt and response tokens, then averages records. Contribution scores select training responses without weighting this loss. The dataset stays fixed during fitting, and an empty dataset leaves the agent unchanged. Agent learning follows either selection outcome, giving the next state $(S_{t+1},H_{t+1},A_{t+1})=(\bar S_t,\bar H_t,A_{t+1})$. Its updated responses guide exploration in the next round.

\subsection{Learning Guarantee}
\label{sec:hace_theory}

The training distribution can change with model updates and harness selection. To study learning under that changing distribution, we theoretically analyze accepted exposure and the final specialist's loss on the resulting experience. Let $\mu_i\geq0$, with $\sum_{i\in\mathcal I}\mu_i=1$, be fixed source probabilities. Fix a finite nonempty specialist class $\mathcal F$ before exploration and a fixed loss $\ell_i(f)\in[0,1]$ for each $i\in\mathcal I$ and $f\in\mathcal F$. This analysis loss is independent of the harness and may differ from $\ell_{\mathrm{task}}$. The specialist's source risk is $R_\mu(f)=\sum_{i\in\mathcal I}\mu_i\ell_i(f)$.

Let $\mathcal H_t$ be the observed history before round $t$, and let $W_{t,i}$ be the accepted training weight of instance $i$. A nonempty accepted update assigns total weight one uniformly to newly selected instances. All other weights are zero. These weights are known by the end of the round. After $T$ rounds, define cumulative expected accepted exposure and loss by
\begin{equation}
C_T(i)=\sum_{t<T}\mathbb E[W_{t,i}\mid\mathcal H_t],\qquad
\widehat L_T(f)=\sum_{t<T}\sum_{i\in\mathcal I}W_{t,i}\ell_i(f),
\label{eq:hace_exposure}
\end{equation}
where $t<T$ denotes rounds $0,\ldots,T-1$. In $\widehat L_T(S_T)$, all accepted instances are evaluated at the final specialist $S_T\in\mathcal F$, which may depend on the observed history. For $v>0$, let $U_T(v)=\sum_{i:C_T(i)<v\mu_i}\mu_i$ be the source mass with insufficient exposure.

\begin{hacetheorem}[Learning from Evolving Experience]
\label{thm:hace_learning}
For any fixed $\delta\in(0,1)$, with probability at least $1-\delta$ over the learning process, every final specialist $S_T\in\mathcal F$ satisfies, simultaneously for all integers $T\geq1$ and all $v>0$,
\begin{equation}
R_\mu(S_T)\leq U_T(v)+
\frac{\widehat L_T(S_T)+\log(|\mathcal F|/\delta)}{(1-e^{-1})v}.
\label{eq:hace_adaptive_risk}
\end{equation}
On the same event, $R_\mu(S_T)\to0$ whenever $v_T>0$ satisfies $v_T\to\infty$, $U_T(v_T)\to0$, and $\widehat L_T(S_T)/v_T\to0$.
\end{hacetheorem}

The bound separates reaching useful training instances from learning and retaining their information. Harness search, agent exploration, and joint selection determine accepted exposure. Specialist optimization determines the final loss on that experience, including whether later updates preserve earlier learning. Sufficient exposure across the source distribution and small loss relative to that exposure yield a small risk bound, even when selection depends on the full preceding history. Appendix~\ref{app:hace_coverage} gives the complete proof. Proposition~\ref{prop:hace_verification_gain} explains how agent fitting and reference changes affect verification before curriculum selection and acceptance determine exposure.

\section{Experiments}
\label{sec:experiments}

In this section, we empirically evaluate whether HACE can push the capability boundary with the converged specialized foundation models across three important domains (recommendation systems, protein function prediction, and material discovery), compared with other self-improving methods.  

\paragraph{Experimental setups.}
By default, we use Gemma4-E4B \citep{gemmateam2026gemma4} as the LLM agent. We also conduct more experiments with various LLM agents (Gemma4-E2B, Gemma4-26B-A4B) in Section~\ref{sec:scaling}. On the specialized foundation models, for recommendation tasks, we evaluate on MovieLens 1M (ML-1M) \citep{harper2015movielens} and four Amazon categories \citep{mcauley2015amazon} using three mainstream foundation models SASRec \citep{kang2018sasrec}, HSTU \citep{zhai2024hstu}, and Wukong \citep{zhang2024wukong}. For protein function prediction, we use frozen ESM-2 \citep{lin2023esm2} and ProtT5-XL \citep{elnaggar2022prottrans} backbones on CAFA \citep{jiang2016cafa}. For crystal stability prediction in material discovery task, we use CHGNet \citep{deng2023chgnet} on WBM \citep{wang2021wbm,riebesell2025matbench}. 
Base denotes the initial specialists, which are already trained to convergence before coevolution. Unless stated otherwise, reported scores apply the fusion in Equation~\ref{eq:hace_fusion} with the weight selected by cross fitting, for Base and for all methods alike. At round 0 the agent is untrained and cross fitting selects a zero weight.
Table~\ref{tab:studied_models} shows that HACE is agnostic to the specialist architectures and their adaptation scope, including causal attention models, a factorization network, protein language models, and a graph neural potential model.

We compare against popular self-improving methods, including ReST-EM \citep{singh2024beyond}, Meta-Harness \citep{lee2026metaharness}, OpenEvolve \citep{sharma2025openevolve}, DGM \citep{zhang2026dgm}, and TTT-Discover \citep{yuksekgonul2026discover}. The recommendation task performs ranking on candidate items, the protein prediction task predicts the function terms, and the materials discovery task ranks candidate crystal structures by stability within a defined chemical system. These experiments measure improvements in prediction with their pretrained representations.

\begin{table}[t]
\centering
\caption{Main results on recommendation across five datasets and three architectures. Entries report HR@10 $\uparrow$ with standard deviations in subscript. \textbf{Bold} marks the best mean in each row.}
\label{tab:main_recsys}
\hacetablesetup
\setlength{\tabcolsep}{1.3pt}
\begin{tabular*}{\linewidth}{@{\extracolsep{\fill}}l*{7}{c}@{}}
\toprule
Dataset & Base & ReST-EM & Meta-Harness & OpenEvolve & DGM & TTT-Discover & \textbf{HACE} \\
\midrule
\multicolumn{8}{@{}l}{\textbf{SASRec} } \\
\addlinespace[1pt]
ML-1M & \result{0.2557}{.0012} & \result{0.2492}{.0022} & \result{0.2491}{.0017} & \result{0.2421}{.0004} & \result{0.2488}{.0018} & \result{0.2473}{.0013} & \bestresult{0.2632}{.0013} \\
Beauty & \result{0.0599}{.0004} & \result{0.0597}{.0010} & \result{0.0631}{.0007} & \result{0.0638}{.0003} & \result{0.0634}{.0005} & \result{0.0633}{.0004} & \bestresult{0.0672}{.0004} \\
CDs & \result{0.0576}{.0005} & \result{0.0587}{.0010} & \result{0.0604}{.0004} & \result{0.0611}{.0002} & \result{0.0592}{.0008} & \result{0.0587}{.0007} & \bestresult{0.0647}{.0006} \\
Sports & \result{0.0315}{.0003} & \result{0.0325}{.0001} & \result{0.0323}{.0002} & \result{0.0331}{.0004} & \result{0.0333}{.0003} & \result{0.0323}{.0002} & \bestresult{0.0356}{.0003} \\
Toys & \result{0.0596}{.0001} & \result{0.0648}{.0007} & \result{0.0657}{.0008} & \result{0.0662}{.0008} & \result{0.0661}{.0009} & \result{0.0646}{.0008} & \bestresult{0.0706}{.0003} \\
\midrule
\multicolumn{8}{@{}l}{\textbf{HSTU}} \\
\addlinespace[1pt]
ML-1M & \result{0.2604}{.0003} & \result{0.2563}{.0005} & \result{0.2569}{.0007} & \result{0.2573}{.0005} & \result{0.2572}{.0006} & \result{0.2588}{.0013} & \bestresult{0.2685}{.0003} \\
Beauty & \result{0.0626}{.0003} & \result{0.0669}{.0013} & \result{0.0673}{.0013} & \result{0.0668}{.0013} & \result{0.0687}{.0004} & \result{0.0695}{.0003} & \bestresult{0.0731}{.0003} \\
CDs & \result{0.0657}{.0002} & \result{0.0654}{.0006} & \result{0.0656}{.0008} & \result{0.0656}{.0010} & \result{0.0658}{.0006} & \result{0.0673}{.0003} & \bestresult{0.0706}{.0002} \\
Sports & \result{0.0348}{.0003} & \result{0.0346}{.0005} & \result{0.0349}{.0004} & \result{0.0351}{.0003} & \result{0.0372}{.0001} & \result{0.0372}{.0002} & \bestresult{0.0394}{.0001} \\
Toys & \result{0.0673}{.0001} & \result{0.0727}{.0006} & \result{0.0732}{.0003} & \result{0.0726}{.0002} & \result{0.0746}{.0007} & \result{0.0740}{.0004} & \bestresult{0.0788}{.0004} \\
\midrule
\multicolumn{8}{@{}l}{\textbf{Wukong}} \\
\addlinespace[1pt]
ML-1M & \result{0.2559}{.0020} & \result{0.2612}{.0008} & \result{0.2573}{.0013} & \result{0.2582}{.0003} & \result{0.2610}{.0013} & \result{0.2607}{.0013} & \bestresult{0.2703}{.0020} \\
Beauty & \result{0.0734}{.0002} & \result{0.0810}{.0004} & \result{0.0805}{.0004} & \result{0.0795}{.0004} & \result{0.0788}{.0004} & \result{0.0772}{.0004} & \bestresult{0.0856}{.0002} \\
CDs & \result{0.0726}{.0005} & \result{0.0771}{.0002} & \result{0.0798}{.0004} & \result{0.0780}{.0004} & \result{0.0785}{.0004} & \result{0.0770}{.0004} & \bestresult{0.0842}{.0005} \\
Sports & \result{0.0388}{.0005} & \result{0.0408}{.0002} & \result{0.0398}{.0002} & \result{0.0390}{.0002} & \result{0.0410}{.0002} & \result{0.0395}{.0002} & \bestresult{0.0438}{.0007} \\
Toys & \result{0.0759}{.0002} & \result{0.0785}{.0004} & \result{0.0752}{.0004} & \result{0.0793}{.0004} & \result{0.0820}{.0004} & \result{0.0770}{.0004} & \bestresult{0.0866}{.0006} \\
\bottomrule
\end{tabular*}
\end{table}

\paragraph{Evaluation metrics.}
Hit Rate (HR@10) measures the fraction of recommendation queries with the target item among the first ten predictions. NDCG uses binary relevance, discounts each relevant candidate at rank $r$ by $1/\log_2(r+1)$, and normalizes by the ideal ranking within the same reference pool. We average NDCG@5 and NDCG@10 on CAFA, and NDCG@3 and NDCG@5 on WBM according to the baseline implementations. For protein function prediction, $F_{\max}$ is the maximum harmonic mean of precision and recall over prediction thresholds \citep{jiang2016cafa}. For material discovery, we use $F_1$ to measure their harmonic mean for identifying stable structures \citep{riebesell2025matbench}. Higher values indicate better performance for these metrics. Appendix~\ref{app:evaluation_metrics} defines all metrics in detail, including the additional measures reported there. Tables report means with standard deviations on three random seeds. Joint selection and fusion calibration only use the split validation data to avoid data contamination, as defined in Section~\ref{sec:hace_selection}.

\subsection{Main Comparison}
\label{sec:main_comparison}

\begin{table}[t]
\centering
\caption{Main results on protein function prediction (CAFA) and material discovery (WBM). Higher is better for all metrics. Full scientific metrics appear in Tables~\ref{tab:full_protein} and~\ref{tab:full_materials}.}
\label{tab:main_science}
\label{tab:main_protein}\label{tab:main_materials}
\hacetablesetup
\begin{tabular*}{\linewidth}{@{\extracolsep{\fill}}l*{6}{c}@{}}
\toprule
& \multicolumn{2}{c}{CAFA / ESM-2} & \multicolumn{2}{c}{CAFA / ProtT5-XL} & \multicolumn{2}{c}{WBM / CHGNet} \\
\cmidrule(lr){2-3}\cmidrule(lr){4-5}\cmidrule(l){6-7}
Method & NDCG $\uparrow$ & $F_{\max}\uparrow$ & NDCG $\uparrow$ & $F_{\max}\uparrow$ & NDCG $\uparrow$ & $F_1\uparrow$ \\
\midrule
Base & \result{0.6285}{.0019} & \result{0.2218}{.0014} & \result{0.6241}{.0011} & \result{0.2221}{.0005} & \result{0.6388}{.0011} & \result{0.5159}{.0003} \\
ReST-EM & \result{0.6286}{.0018} & \result{0.2219}{.0014} & \result{0.6242}{.0011} & \result{0.2221}{.0005} & \result{0.6496}{.0006} & \result{0.5508}{.0004} \\
Meta-Harness & \result{0.6309}{.0015} & \result{0.2224}{.0013} & \result{0.6261}{.0017} & \result{0.2229}{.0006} & \result{0.6513}{.0008} & \result{0.5510}{.0004} \\
OpenEvolve & \result{0.6320}{.0018} & \result{0.2230}{.0012} & \result{0.6254}{.0014} & \result{0.2230}{.0003} & \result{0.6486}{.0008} & \result{0.5463}{.0003} \\
DGM & \result{0.6325}{.0019} & \result{0.2219}{.0014} & \result{0.6273}{.0016} & \result{0.2235}{.0006} & \result{0.6496}{.0007} & \result{0.5508}{.0003} \\
TTT-Discover & \result{0.6287}{.0018} & \result{0.2219}{.0014} & \result{0.6241}{.0011} & \result{0.2221}{.0005} & \result{0.6474}{.0007} & \result{0.5504}{.0003} \\
\midrule
\textbf{HACE} & \bestresult{0.6745}{.0016} & \bestresult{0.2438}{.0019} & \bestresult{0.6739}{.0013} & \bestresult{0.2317}{.0005} & \bestresult{0.6541}{.0010} & \bestresult{0.5513}{.0004} \\
\bottomrule
\end{tabular*}
\end{table}

\paragraph{Recommendation.}
HACE achieves the highest mean HR@10 in all 15 settings in Table~\ref{tab:main_recsys}. Relative gains over Base range from $2.9\%$ to $18.5\%$, averaging $12.1\%$ across settings. On ML-1M, several evolution methods fall below the initial SASRec and HSTU models, while HACE improves all three architectures. On the four Amazon categories, Wukong has the strongest initial scores among the three architectures, yet HACE further improves its HR@10 by $12.9\%$ to $16.6\%$. Contribution rewards train the agent against the current specialist throughout this process.

\paragraph{Protein function prediction.}
Table~\ref{tab:main_science} shows improvements with both protein backbones. HACE raises NDCG from $0.6285$ to $0.6745$ with ESM-2 and from $0.6241$ to $0.6739$ with ProtT5-XL. Existing evolution methods improve over Base by at most $0.0040$ and $0.0032$, respectively, compared with $0.0460$ and $0.0498$ for HACE. HACE also leads in $F_{\max}$, AUPR, and Recall@10. These gains cover both the ordering of function terms and their classification. Table~\ref{tab:full_protein} also reports semantic distance, for which DGM has the lowest mean on ProtT5-XL. The ablations below examine how these gains depend on harness search and specialist optimization.

\paragraph{Material Discovery.}
On WBM, all competing methods improve NDCG over Base, from $0.6388$ to at least $0.6474$. HACE reaches $0.6541$, exceeding the strongest competing mean by $0.0028$. HACE also achieves the best mean on all four additional metrics in Table~\ref{tab:full_materials}. Removing specialist optimization reduces NDCG to $0.6395$, close to Base at $0.6388$. The contribution of supervised specialist updates is therefore particularly clear on this task. Supervised adaptation improves numerical predictions, while the ranking objective favors recovering stable candidates near the top.

\subsection{Ablation Studies}
\label{sec:ablation}

\begin{table}[!htbp]
\centering
\caption{Ablation studies on recommendation (Toys/SASRec), protein function prediction (CAFA/ESM-2 650M), and material discovery (WBM/CHGNet). Metrics are HR@10 for Toys and mean NDCG for CAFA and WBM. }
\label{tab:ablation}
\hacetablesetup
\setlength{\tabcolsep}{7pt}
\begin{tabular}{@{}lccc@{}}
\toprule
Leave-one-out component & Toys & CAFA & WBM \\
\midrule
\textbf{HACE}
& \bestresult{0.0706}{.0003}
& \bestresult{0.6745}{.0016}
& \bestresult{0.6541}{.0010} \\
\midrule
w/o harness optimization
& \result{0.0658}{.0005}
& \result{0.6307}{.0019}
& \result{0.6500}{.0012} \\

w/o specialist optimization
& \result{0.0641}{.0005}
& \result{0.6310}{.0022}
& \result{0.6395}{.0012} \\

w/o agent optimization
& \result{0.0622}{.0006}
& \result{0.6520}{.0016}
& \result{0.6502}{.0011} \\

w/o joint selection
& \result{0.0650}{.0006}
& \result{0.6531}{.0020}
& \result{0.6497}{.0013} \\

w/o agent-guided experience
& \result{0.0662}{.0005}
& \result{0.6575}{.0026}
& \result{0.6507}{.0012} \\
\bottomrule
\end{tabular}
\end{table}

We conduct leave-one-out ablation studies in Table~\ref{tab:ablation} in three domains. Removing agent optimization causes the largest drop on Toys, from $0.0706$ to $0.0622$. On CAFA, removing harness or specialist optimization reduces NDCG by $0.0438$ and $0.0435$, respectively. Harness search controls the information exposed to the agent and the experience selected for adaptation, while specialist optimization learns from that experience using task labels. Agent optimization matters most on Toys and specialist optimization on WBM, indicating that their contributions vary by task. Removing agent-guided experience lowers performance even when specialist optimization remains enabled, by $0.0044$ on Toys, $0.0170$ on CAFA, and $0.0034$ on WBM. This decrease supports the role of agent exploration in selecting training experience, in addition to its role in prediction. Removing joint selection also reduces all three scores. Together, we show that all designed components contribute positively to the framework.

\subsection{Agent and Model Scaling}
\label{sec:scaling}

\begin{table}[!htbp]
\centering
\caption{Scaling experiment on protein function prediction task. NDCG averages
cutoffs 5 and 10. Base denotes the initial specialists, which are already trained to convergence before coevolution; Appendix~\ref{app:experimental_details} states the criterion we use. Agent variants use a shared 30-round configuration.
Subscripts report standard deviations of within-run gains on three random seeds.}
\label{tab:model_scaling}
\hacetablesetup
\setlength{\tabcolsep}{7pt}
\begin{tabular}{@{}lccc@{}}
\toprule
Model & Base & HACE & Reported $\Delta$ \\
\midrule
\multicolumn{4}{@{}l}{\textbf{Specialist size}\quad Fixed agent Gemma-4 E4B} \\

ESM-2 35M       & $0.6283$ & $0.6692$ & $+0.0409_{\color{black!50}\pm 0.0007}$ \\
ESM-2 150M      & $0.6211$ & $0.6633$ & $+0.0422_{\color{black!50}\pm 0.0004}$ \\
ESM-2 650M      & $0.6285$ & $0.6745$ & $+0.0460_{\color{black!50}\pm 0.0006}$ \\
ESM-2 3B        & $0.6339$ & $0.6763$ & $+0.0424_{\color{black!50}\pm 0.0007}$ \\
\midrule
\multicolumn{4}{@{}l}{\textbf{Agent variant}\quad Fixed specialist ESM-2 650M} \\

Gemma-4 E2B     & $0.6285$ & $0.6297$ & $+0.0012_{\color{black!50}\pm 0.0018}$ \\
Gemma-4 E4B     & $0.6285$ & $0.6745$ & $+0.0460_{\color{black!50}\pm 0.0006}$ \\
Gemma-4 26B-A4B & $0.6285$ & $0.6339$ & $+0.0054_{\color{black!50}\pm 0.0014}$ \\
\bottomrule
\end{tabular}
\end{table}

In the specialist sweep with the agent fixed, HACE improves prediction with all four ESM-2 sizes in Table~\ref{tab:model_scaling}, from 35M to 3B parameters. The absolute NDCG gains remain between $0.0409$ and $0.0460$. The largest final score comes from the 3B model, while the largest gain occurs at 650M. The 35M model after HACE also exceeds the initial 3B model, with NDCG of $0.6692$ versus $0.6339$. 
In the agent comparison, E4B achieves the largest mean gain of $0.0460$, compared with $0.0012$ for E2B and $0.0054$ for 26B-A4B. These are within-run gains from a separate matched 30-round sweep with ESM-2 650M, detailed in Appendix~\ref{app:experimental_details}. The gains do not increase monotonically with agent size, suggesting that a larger agent does not necessarily contribute more to the specialist in this setting.

\subsection{Evolution Dynamics}
\label{sec:evolution_analysis}


\begingroup
\setlength{\intextsep}{6pt plus 2pt minus 2pt}
\begin{figure}[!htb]
\centering
\includegraphics[width=.90\linewidth]{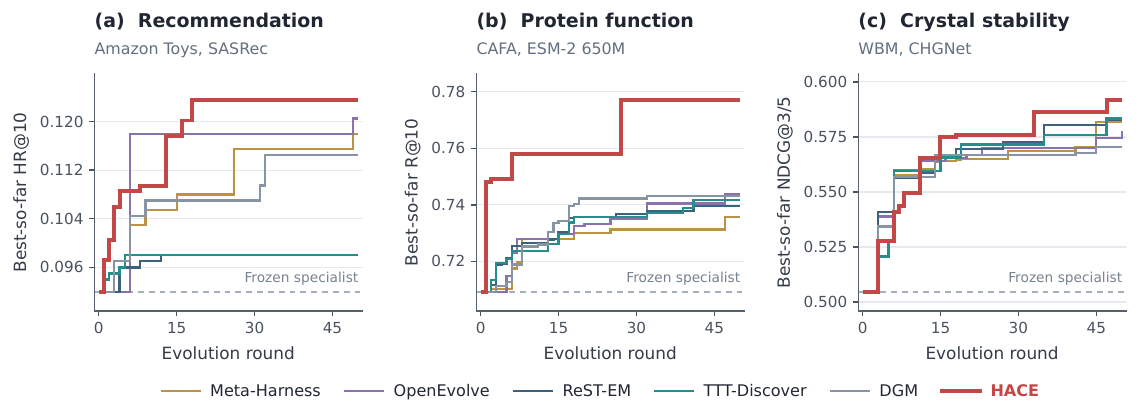}
\caption{Evolution dynamics of three domains. Best recorded scores during evolution, measured by HR@10, Recall@10, and mean NDCG@3 and NDCG@5. Dashed lines mark the frozen specialist.}
\label{fig:evolution_source}
\end{figure}

\paragraph{Improvement over rounds.}
Figure~\ref{fig:evolution_source} shows HACE reaching the highest score in all three recorded trajectories. The best recorded recommendation score rises mostly in early rounds, while protein and materials reach new maxima after longer intervals. In HACE, each retained specialist and harness defines the starting point for the next round. Fresh agent responses are verified against the current specialist. This update order makes agent supervision depend on the retained specialist. The trajectories show when higher scoring configurations are found during search.

\begin{figure}[!htb]
\centering
\includegraphics[width=.95\linewidth]{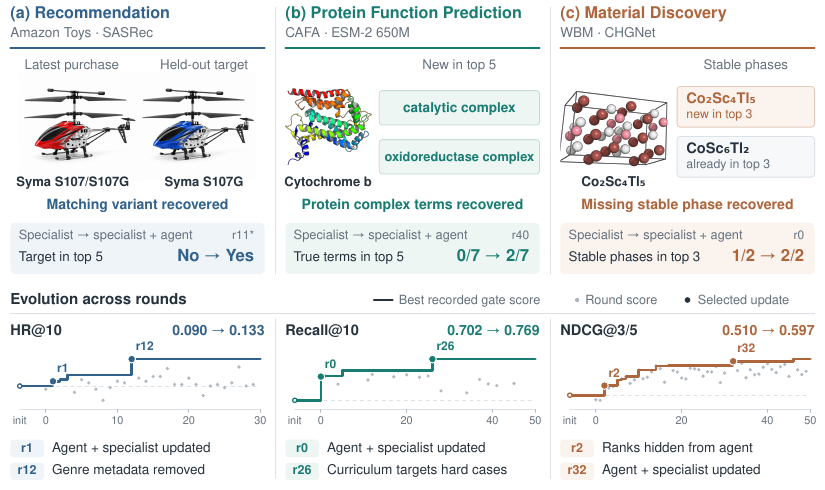}
\caption{Illustrations on events during coevolution. Upper panels compare specialist and combined predictions. Lower panels show recorded selection scores and their running maxima from one run in each domain, with selected updates annotated.}
\label{fig:case_study}
\end{figure}
\endgroup

\paragraph{Case studies on three domains.}
Figure~\ref{fig:case_study} compares specialist and combined predictions on three domains. HACE recovers a matching product variant, places two relevant protein terms in the top five, and brings a second stable phase into the top three. Recorded changes include removing genre metadata, targeting difficult training instances, and hiding candidate ranks from the agent. These examples illustrate how the harness search and coevolution change both the context used for prediction and the experience used for reciprocal learning.

\section{Conclusion}
\label{sec:conclusion}

We introduced HACE, a novel framework for reciprocal self-improvement between specialized foundation models and LLM agents. By evolving both the learners and their interaction with an updating harness, HACE makes each model a source of learning for the other. Our analysis identifies sufficient conditions for specialist learning from adaptively selected experience. Experiments in three domains demonstrate improvements beyond initially converged specialists across different architectures. These results show how heterogeneous AI models can support each other's development. Future work can study sustained coevolution, collaboration among multiple specialists, and autonomous model development for scientific discovery.



\FloatBarrier
\bibliographystyle{assets/plainnat}
\bibliography{bibs/wenqian}

\clearpage
\beginappendix

\section{Limitations}

HACE coevolves a single specialized foundation model with a single LLM agent. We do not study coevolution among multiple specialists or agents, which may require different mechanisms for interaction and learning. The compared methods differ in which components they can update. Some support parameter updates, others evolve the harness or agent program with the specialist held fixed, and several assume a closed-weight LLM. We follow each method's original formulation rather than introducing specialist adaptation beyond its design. These comparisons therefore evaluate the methods as complete systems, while Table~\ref{tab:ablation} separately examines the contribution of each HACE component. Our findings are based on the evaluated domains, architectures, and computational budgets. Further work can study how long coevolution sustains useful learning and how to allocate computation among harness search, specialist adaptation, and agent training.

\section{Proofs and Implementation Details}
\label{app:hace_theory}

Throughout this appendix, $\log$ is the natural logarithm and an explicitly written $\log_2$ uses base two. For a finite set or indexed dataset, $|B|$ counts its elements or records, while $|z|$ counts tokens. The indicator $\mathbf 1\{B\}$ is one when condition $B$ holds and zero otherwise. Empty sums are zero and empty products are one. Unless specified otherwise, $\sum_i$ means $\sum_{i\in\mathcal I}$ and $\sum_{t<T}$ means $\sum_{t=0}^{T-1}$.

\subsection{Specialist Learning}
\label{app:hace_coverage}

Let $\mathcal D_{\mathrm{train}}=\{(x_i,y_i):i\in\mathcal I\}$ be the labeled training pool defined in Section~\ref{sec:hace_setup}. Equation~\ref{eq:hace_model_loss} gives the supervised objective on the selected training indices $\mathcal I_t^S$. The specialist's state includes its trainable parameters $\phi$ and stored score calibration. Adaptation starts from $S_t$ and uses a fixed optimization budget.

When supervised replay is used, the following term is added to Equation~\ref{eq:hace_model_loss}, with both terms using original task labels,
\begin{equation}
\mathcal R_{\mathrm{replay},t}(\phi)=\sum_{i\in\mathcal I}\gamma_{t,i}
\ell_{\mathrm{task}}(S_\phi(x_i),y_i).
\label{eq:hace_replay_loss}
\end{equation}
The nonnegative weights $\gamma_{t,i}$ are fixed during the specialist update and are zero unless instance $i$ is retained from earlier specialist training. Their sum sets the total replay strength, and equal positive weights give uniform replay over the retained instances. Setting all weights to zero disables replay. This defines the term without assuming a particular retention policy.

For the proof of Theorem~\ref{thm:hace_learning}, condition on the fixed source pool and initial configuration. The source probabilities $\mu_i$, specialist class $\mathcal F$, and losses $\ell_i(f)$ are fixed before exploration. The analysis loss may differ from the optimization loss $\ell_{\mathrm{task}}$. Interpret $\mathcal H_t$ as the sigma field of information available before round $t$. These histories increase with $t$, and $\mathcal H_{t+1}$ includes that round's proposal, responses, specialist training, validation decision, and agent update.

The accepted weights are measurable with respect to $\mathcal H_{t+1}$, are nonnegative, and satisfy $\sum_iW_{t,i}\leq1$. Under the uniform allocation used here, $W_{t,i}=1/|\mathcal I_t^S|$ on selected instances in a nonempty accepted update and zero otherwise. Conditional expectation given $\mathcal H_t$ averages over the remaining round randomness, allowing dependence between selection and acceptance. The final specialist may depend on the full observed history. Replay may affect its final loss but is not counted as additional exposure. We use $C_T(i)$, $\widehat L_T(f)$, and $U_T(v)$ as defined in Section~\ref{sec:hace_theory}.

\begin{proof}[Proof of Theorem~\ref{thm:hace_learning}]
Fix a predictor $f\in\mathcal F$ before the random exploration history. Define its accepted loss in round $t$ and the conditional mean of that loss by
\begin{equation}
X_t(f)=\sum_iW_{t,i}\ell_i(f),\qquad
m_t(f)=\mathbb E[X_t(f)\mid\mathcal H_t].
\label{eq:hace_one_round_loss}
\end{equation}
The weight and loss bounds give $0\leq X_t(f)\leq1$ and $0\leq m_t(f)\leq1$. Since the task losses are fixed, linearity of conditional expectation gives
\begin{equation}
\sum_{t<T}m_t(f)=\sum_iC_T(i)\ell_i(f),\qquad
\sum_{t<T}X_t(f)=\widehat L_T(f).
\label{eq:hace_predictable_loss}
\end{equation}

We first compare the conditional expected losses with their realized sum, uniformly over rounds. For $0\leq x\leq1$, convexity gives $e^{-x}\leq1-(1-e^{-1})x$. Taking conditional expectations and using $1-u\leq e^{-u}$ yields
\begin{equation}
\mathbb E[e^{-X_t(f)}\mid\mathcal H_t]
\leq1-(1-e^{-1})m_t(f)
\leq e^{-(1-e^{-1})m_t(f)}.
\label{eq:hace_conditional_exponential}
\end{equation}
Define $M_0(f)=1$ and
\begin{equation}
M_T(f)=\exp\left((1-e^{-1})\sum_{t<T}m_t(f)-\sum_{t<T}X_t(f)\right).
\label{eq:hace_supermartingale}
\end{equation}
This process is nonnegative and measurable with respect to $\mathcal H_T$. Equation~\ref{eq:hace_conditional_exponential} implies
\begin{equation}
\begin{aligned}
\mathbb E[M_{T+1}(f)\mid\mathcal H_T]
&=M_T(f)e^{(1-e^{-1})m_T(f)}
\mathbb E[e^{-X_T(f)}\mid\mathcal H_T]\\
&\leq M_T(f).
\end{aligned}
\label{eq:hace_supermartingale_step}
\end{equation}
Thus it is a supermartingale, meaning that its conditional expected next value is at most its current value. This exponential construction gives a probability bound valid over time~\citep{howard2021time}. The following stopping argument establishes that bound directly.

For a threshold $a>1$, let $\sigma_a$ be the first integer $T\geq0$ with $M_T(f)\geq a$, and set $\sigma_a=\infty$ if no such round exists. Whether a crossing has occurred by round $t$ is known from $\mathcal H_t$, so $\sigma_a$ is a stopping time. For a positive integer $n$, write $\sigma_a\wedge n=\min\{\sigma_a,n\}$. The stopped process satisfies
\begin{equation}
M_{\sigma_a\wedge n}(f)
=1+\sum_{t=0}^{n-1}\mathbf 1\{\sigma_a>t\}
\big(M_{t+1}(f)-M_t(f)\big).
\label{eq:hace_stopped_process}
\end{equation}
Each indicator is known under $\mathcal H_t$, so each expected summand is nonpositive by Equation~\ref{eq:hace_supermartingale_step}. Hence $\mathbb E[M_{\sigma_a\wedge n}(f)]\leq1$. The process is at least $a$ whenever a crossing occurs by $n$ and is nonnegative otherwise. Therefore $a\Pr(\sigma_a\leq n)\leq1$. Letting $n\to\infty$ gives
\begin{equation}
\Pr\big(\exists\,T\geq0\text{ with }M_T(f)\geq a\big)\leq\frac1a.
\label{eq:hace_maximal_bound}
\end{equation}
Choose $a=|\mathcal F|/\delta$ and take a union bound over the fixed class. With probability at least $1-\delta$, no predictor's process crosses this threshold at any round. On this event, taking logarithms and applying Equation~\ref{eq:hace_predictable_loss} gives, for all $f\in\mathcal F$ and $T\geq1$,
\begin{equation}
(1-e^{-1})\sum_iC_T(i)\ell_i(f)
\leq\widehat L_T(f)+\log(|\mathcal F|/\delta).
\label{eq:hace_exposure_bound}
\end{equation}
No independence across rounds is used.

This bound controls task risk wherever accepted exposure is sufficiently large. Fix any $v>0$ on this event. Instances with $C_T(i)<v\mu_i$ contribute at most $U_T(v)$ to the source risk, because their losses are at most one. For each remaining instance, $\mu_i\leq C_T(i)/v$. Summing and using Equation~\ref{eq:hace_exposure_bound} gives
\begin{equation}
R_\mu(f)\leq U_T(v)+\frac1v\sum_iC_T(i)\ell_i(f)
\leq U_T(v)+\frac{\widehat L_T(f)+\log(|\mathcal F|/\delta)}{(1-e^{-1})v}.
\label{eq:hace_risk_decomposition}
\end{equation}
This deterministic implication holds for every $v>0$ on the same event. Since that event already covers every predictor and round, we may substitute the adaptive final specialist $S_T\in\mathcal F$. The result also holds at any finite stopping time chosen from the observed history. Treating only the retrospectively selected $S_T$ as fixed would not justify the probability bound.

Finally, for any positive levels $v_T\to\infty$ such that $U_T(v_T)\to0$ and $\widehat L_T(S_T)/v_T\to0$, all terms in the upper bound tend to zero. Nonnegativity then gives $R_\mu(S_T)\to0$. The levels may depend on the observed history because the same event covers every $v$. This conclusion applies to trajectories satisfying these conditions. For a fixed finite class, a zero risk limit requires $\min_{f\in\mathcal F}R_\mu(f)=0$.
\end{proof}

The exposure term connects the bound to the training allocation. When $v$ is proportional to $T$, the term from class size decreases as $1/T$, while the other terms depend on exposure and retained loss. To relate exposure to selection, suppose the nonempty training set has at most $B\geq1$ instances and uses uniform weights. On the joint event that instance $i$ is selected and the proposal is accepted, its weight is at least $1/B$. Thus
\begin{equation}
\mathbb E[W_{t,i}\mid\mathcal H_t]
\geq\frac1B\Pr\big(i\in\mathcal I_t^S,\ \Delta J_t>0\mid\mathcal H_t\big).
\label{eq:hace_inclusion_exposure}
\end{equation}
Summing gives the corresponding lower bound for $C_T(i)$. The joint probability preserves dependence between inclusion and acceptance. A high overall acceptance rate need not expose every instance, and verification need not lead to inclusion after curriculum truncation.

Small losses immediately after each update do not ensure small final loss when later updates forget those instances. Replay may affect this final loss. If acceptance has conditional probability zero in every round, then $C_T(i)=0$ and $U_T(v)=1$, so the bound gives no learning conclusion. Expected exposure is generally not observable from one trace, and $|\mathcal F|$ counts a class fixed before exploration rather than adaptively visited checkpoints.

A population bound requires assumptions on how the source pool is drawn. Let $P$ be a distribution of labeled task instances. Suppose the source pool consists of $n=|\mathcal I|\geq1$ independent draws $Z_i=(x_i,y_i)$ from $P$, indexed by $i=1,\ldots,n$, and set $\mu_i=1/n$. Assume $\mathcal F$ is fixed independently of these draws. For a generic draw $Z\sim P$, let $\ell(f,Z)\in[0,1]$ be a fixed loss and set $\ell_i(f)=\ell(f,Z_i)$. Define
\begin{equation}
R_P(f)=\mathbb E_{Z\sim P}[\ell(f,Z)],\qquad
R_\mu(f)=\frac1n\sum_{i=1}^n\ell(f,Z_i).
\end{equation}
Under the remaining assumptions of Theorem~\ref{thm:hace_learning}, with probability at least $1-\delta$ over the pool and learning process, every final specialist satisfies, simultaneously for all $T\geq1$ and $v>0$,
\begin{equation}
R_P(S_T)\leq U_T(v)
+\frac{\widehat L_T(S_T)+\log(2|\mathcal F|/\delta)}{(1-e^{-1})v}
+\sqrt{\frac{\log(2|\mathcal F|/\delta)}{2n}}.
\label{eq:hace_population_risk}
\end{equation}

\textit{Proof.}
For this calculation, let $\xi$ be a centered random variable supported in an interval of width one and let $\eta\geq0$ be an exponential moment parameter. Define $\psi(\eta)=\log\mathbb E[e^{\eta\xi}]$. Differentiation shows that $\psi''(\eta)$ is the variance of $\xi$ under probabilities reweighted by $e^{\eta\xi}$. This distribution retains the same support, and its variance is at most the expected squared distance to the interval midpoint, hence at most $1/4$. Since $\psi(0)=\psi'(0)=0$, integrating twice gives
\begin{equation}
\psi(\eta)=\int_0^\eta(\eta-s)\psi''(s)\,ds\leq\eta^2/8,
\end{equation}
where $s$ is the integration variable. Apply this inequality to $R_P(f)-\ell(f,Z_i)$ for a fixed $f$. Independence and exponential Markov inequality give, for every $\epsilon>0$ and $\eta>0$,
\begin{equation}
\Pr(R_P(f)-R_\mu(f)\geq\epsilon)
\leq\exp(-\eta n\epsilon+n\eta^2/8).
\end{equation}
Choosing $\eta=4\epsilon$ gives an upper bound of $e^{-2n\epsilon^2}$. Taking $\epsilon=\sqrt{\log(2|\mathcal F|/\delta)/(2n)}$ and a union bound over $\mathcal F$ therefore yields, with probability at least $1-\delta/2$,
\begin{equation}
R_P(f)\leq R_\mu(f)+\sqrt{\frac{\log(2|\mathcal F|/\delta)}{2n}}
\quad\text{for every }f\in\mathcal F.
\label{eq:hace_pool_generalization}
\end{equation}
Conditional on the realized pool and initial configuration, apply Theorem~\ref{thm:hace_learning} with failure probability $\delta/2$. Averaging over the pool and initial configuration keeps its failure probability at most $\delta/2$. A union bound combines this event with Equation~\ref{eq:hace_pool_generalization}, without requiring independence between the events. On their intersection, which has probability at least $1-\delta$, adding the bounds proves Equation~\ref{eq:hace_population_risk} for every $T$ and $v$.

Repeated training on a fixed pool does not reduce the final sample term. Population consistency additionally requires that term to vanish, for example by increasing independent source data while controlling class size. The source risk result permits a class chosen after fixing the pool but before exploration. This population extension requires the stronger independence condition stated above.

\subsection{Agent Learning and Verification}
\label{app:hace_verification}

We next analyze the responses that provide learning experience, conditioning on the collection state and using the fixed source probabilities $\mu_i$. Write $A_\theta$ for the agent language model with trainable parameters $\theta$, including an adapter when used, and let $\theta_t$ be the parameters of $A_t$. For a fixed specialist and harness, $p_i$ is the nonempty prompt constructed from $x_i$ and specialist predictions, excluding $y_i$. The decoded distribution $\pi_\theta(\tau\mid p_i)$ gives the probability of the complete response sequence $\tau$ after decoding and stopping. We collect $m\geq1$ responses per instance, where $m$ is an integer. They are independent conditional on the prompt and any shared task context.

Fix a collection round $t$, its specialist and harness, and the task context used to form each prompt. For an instance $i$ with evaluable supervision, define
\begin{equation}
E_t(i)=\{\tau\in\mathcal T\mid \tau\text{ is valid and }r_i(\tau)>0\},\qquad
\alpha_t(i)=\pi_{\theta_t}(E_t(i)\mid p_i).
\label{eq:hace_verified_response}
\end{equation}
Here $\mathcal T$ is the countable space of finite response token sequences. For any response distribution $\pi$, the notation $\pi(E)$ means $\sum_{\tau\in E}\pi(\tau)$. For an instance without evaluable supervision, set $E_t(i)=\varnothing$ and $\alpha_t(i)=0$.

Let $I$ be the source instance index, with $\Pr(I=i)=\mu_i$, and let $V$ be the event that at least one of its $m$ collected responses passes verification. Conditional independence of those responses gives $\Pr(V\mid I=i)=1-(1-\alpha_t(i))^m$. Therefore, whenever $\Pr(V)>0$, Bayes' rule gives
\begin{equation}
\Pr(I=i\mid V)=
\frac{\mu_i[1-(1-\alpha_t(i))^m]}
{\sum_j\mu_j[1-(1-\alpha_t(j))^m]}.
\label{eq:hace_verified_distribution}
\end{equation}
This describes source sampling followed only by verification. A zero denominator means that no source instance can provide a verified response under this collection procedure.

If the responses share a random context $C$, independence must instead hold conditional on that context. Define $\alpha_t(i,C)$ by Equation~\ref{eq:hace_verified_response} using the prompt and verifier determined by $C$. The collection probability becomes
\begin{equation}
\Pr(V\mid I=i)=\mathbb E\big[1-(1-\alpha_t(i,C))^m\mid I=i\big].
\label{eq:hace_context_collection}
\end{equation}
Equation~\ref{eq:hace_verified_distribution} then uses this conditional expectation in place of each bracketed collection probability. For $m>1$, the function $a\mapsto1-(1-a)^m$ is concave on $[0,1]$, since its second derivative is $-m(m-1)(1-a)^{m-2}$ in the interior. Thus Jensen's inequality yields
\begin{equation}
\Pr(V\mid I=i)\leq1-\big(1-\mathbb E[\alpha_t(i,C)\mid I=i]\big)^m.
\end{equation}
Averaging the response probability before computing the collection probability can therefore overestimate the chance of verification. Equality holds for $m=1$.

We next compare verification before agent fitting, after fitting under the same reference, and after the next specialist and harness selection. Fix one instance and its agent collection context under the retained pair $(\bar S_t,\bar H_t)$, and suppress the instance index. Write $E_t$ for its verified response set, and let $\pi_t$ and $\widetilde\pi_t$ be the decoded response distributions before and after agent fitting under that same prompt, specialist, and harness. Set $\alpha_t=\pi_t(E_t)>0$ and $\widetilde\alpha_t=\widetilde\pi_t(E_t)>0$.

Let $q_t$ be a distribution on selected responses in $E_t$, including their training weights. Every response with positive $q_t$ mass must have positive probability under both $\pi_t$ and $\widetilde\pi_t$. For distributions $q$ and $p$, define $\mathrm{KL}(q\Vert p)=\sum_{\tau:q(\tau)>0}q(\tau)\log[q(\tau)/p(\tau)]$, with value $+\infty$ if $p$ is zero on a response with positive $q$ mass. All divergences and log likelihood expectations below are assumed finite. Define
\begin{equation}
g_t=\mathbb E_{q_t}\log\frac{\widetilde\pi_t(\tau)}{\pi_t(\tau)},\qquad
k_t=\mathrm{KL}\big(q_t\Vert\pi_t(\cdot\mid E_t)\big),
\label{eq:hace_agent_gain_definitions}
\end{equation}
where $\pi_t(\tau\mid E_t)=\pi_t(\tau)/\alpha_t$ on $E_t$ and is zero elsewhere. The conditional distribution after fitting is defined in the same way. Thus $g_t$ measures the mean decoded likelihood gain on selected responses, while $k_t$ measures how their distribution differs from the old distribution conditional on verification.

The next agent collection uses the fitted agent after the next specialist and harness selection. Let $\pi_{t+1}$ and $E_{t+1}$ be its decoded distribution and verified response set for the same instance, and set $\alpha_{t+1}=\pi_{t+1}(E_{t+1})$. These probabilities condition on the realized parameters, specialist, harness, and collection context. Define
\begin{equation}
\eta_t=\mathrm{TV}(\widetilde\pi_t,\pi_{t+1})
+\widetilde\pi_t(E_t\setminus E_{t+1}),
\label{eq:hace_reference_change}
\end{equation}
where $\mathrm{TV}(p,q)=\sup_{B\subseteq\mathcal T}|p(B)-q(B)|$. The two terms measure the change in response distribution and the old probability of responses that no longer pass verification.

\begin{haceproposition}[Agent Verification]
\label{prop:hace_verification_gain}
Under the conditions above,
\begin{equation}
\log\frac{\widetilde\alpha_t}{\alpha_t}
=g_t-k_t+\mathrm{KL}\big(q_t\Vert\widetilde\pi_t(\cdot\mid E_t)\big)
\geq g_t-k_t,
\label{eq:hace_agent_verification_gain}
\end{equation}
and
\begin{equation}
\alpha_{t+1}\geq\big[\alpha_t e^{g_t-k_t}-\eta_t\big]_+,
\label{eq:hace_next_verification}
\end{equation}
where $[x]_+=\max\{x,0\}$. In particular, $g_t>k_t$ is sufficient for improved verification under the fixed reference. The stronger condition $\eta_t<\alpha_t(e^{g_t-k_t}-1)$ is sufficient for $\alpha_{t+1}>\alpha_t$ after the reference changes.
\end{haceproposition}

\begin{proof}
Expanding the conditional distributions gives
\begin{align*}
k_t&=\mathbb E_{q_t}\log\frac{q_t(\tau)}{\pi_t(\tau)}+\log\alpha_t,\\
\mathrm{KL}\big(q_t\Vert\widetilde\pi_t(\cdot\mid E_t)\big)
&=\mathbb E_{q_t}\log\frac{q_t(\tau)}{\widetilde\pi_t(\tau)}+\log\widetilde\alpha_t.
\end{align*}
Subtracting the first equality from the second proves the identity in Equation~\ref{eq:hace_agent_verification_gain}. To prove its inequality, concavity of the logarithm gives, for any distributions $q$ and $p$ with the stated support and integrability conditions,
\begin{equation}
\mathbb E_q\log\frac{p(\tau)}{q(\tau)}
\leq\log\sum_{\tau:q(\tau)>0}p(\tau)\leq0.
\end{equation}
Hence the KL divergence is nonnegative. Exponentiating yields $\widetilde\alpha_t\geq\alpha_t e^{g_t-k_t}$.

For the next collection, total variation and the decomposition of $E_t$ imply
\begin{align*}
\alpha_{t+1}
&\geq\widetilde\pi_t(E_{t+1})-\mathrm{TV}(\widetilde\pi_t,\pi_{t+1})\\
&\geq\widetilde\alpha_t-\widetilde\pi_t(E_t\setminus E_{t+1})
-\mathrm{TV}(\widetilde\pi_t,\pi_{t+1})\\
&=\widetilde\alpha_t-\eta_t.
\end{align*}
Combining this inequality with the fixed reference bound and probability nonnegativity proves Equation~\ref{eq:hace_next_verification}. The two strict improvement conditions follow directly.
\end{proof}

If $q_t=\pi_t(\cdot\mid E_t)$, then $k_t=0$. Response selection and length weights generally change this distribution. For $m$ conditionally independent responses in the next collection, monotonicity of $a\mapsto1-(1-a)^m$ also gives
\begin{equation}
\begin{aligned}
&\Pr(\text{at least one verified response in round }t+1)\\
&\qquad\geq1-\left(1-\big[\alpha_t e^{g_t-k_t}-\eta_t\big]_+\right)^m.
\end{aligned}
\label{eq:hace_next_collection}
\end{equation}
The probability conditions on the fixed instance, collection state, and context. The lower bound inside the parentheses lies in $[0,1]$ because $\eta_t\geq0$ and $\alpha_t e^{g_t-k_t}\leq\widetilde\alpha_t\leq1$. With random shared context, apply this inequality conditionally and then average as in Equation~\ref{eq:hace_context_collection}.

To connect this result to the training update, we decompose Equation~\ref{eq:hace_agent_loss} into its prompt and response terms. The dataset $\mathcal Z_t$, token indices, and raw next-token probabilities $p_\theta$ are defined in Section~\ref{sec:hace_agent}. The loss includes prompt tokens after the first token and all response tokens present in the record, including a stopping token when present. Contribution scores select records without weighting their losses, and the dataset remains fixed during fitting.

Separating prompt and response tokens shows which likelihood gains this objective measures. Fix a nonempty training dataset $\mathcal Z$ and suppress the update round. Let $\theta$ and $\widetilde\theta$ be the parameters before and after fitting. Its indexed records are $z_j=p_j\Vert\tau_j$ for $j=1,\ldots,|\mathcal Z|$, where $\Vert$ denotes token concatenation. Each record has at least two tokens. Let $w_j=1/(|z_j|-1)$. The prompt $p_j$ is nonempty, and $p_{j,\ell}$ and $p_{j,<\ell}$ denote its token at position $\ell$ and preceding prefix. The same indexing applies to $z_j$. Define the raw response token product
\begin{equation}
P_\theta(\tau_j\mid p_j)
=\prod_{\ell=|p_j|+1}^{|z_j|}p_\theta(z_{j,\ell}\mid z_{j,<\ell}).
\end{equation}
Splitting the predicted tokens at the prompt boundary rewrites Equation~\ref{eq:hace_agent_loss} exactly as
\begin{equation}
\mathcal L_A(\theta)=-\frac1{|\mathcal Z|}\sum_jw_j
\left[\sum_{\ell=2}^{|p_j|}\log p_\theta(p_{j,\ell}\mid p_{j,<\ell})
+\log P_\theta(\tau_j\mid p_j)\right].
\label{eq:hace_loss_split}
\end{equation}
An empty prompt token sum is zero. The first token has no preceding token in the training sequence and is excluded from the implemented loss. Assume that both parameter values assign positive probabilities to every predicted token in these records. The dataset and its length weights remain fixed during this comparison.

Partition the record indices into groups $\mathcal J_u$ with the same source instance, prompt, and verifier, where $u$ ranges over the nonempty groups. Write $p_u$ for their common prompt. Their weighted response distributions and shares of total weight are
\begin{equation}
q_u(\tau)=\frac{\sum_{j\in\mathcal J_u,\,\tau_j=\tau}w_j}{\sum_{j\in\mathcal J_u}w_j},\qquad
\rho_u=\frac{\sum_{j\in\mathcal J_u}w_j}{\sum_jw_j}.
\label{eq:hace_effective_response_distribution}
\end{equation}
Thus $\sum_u\rho_u=1$, and each $q_u$ has finite support consisting of responses that pass one fixed verifier. Subtracting the two losses and collecting response terms by group gives
\begin{equation}
\begin{aligned}
\mathcal L_A(\theta)-\mathcal L_A(\widetilde\theta)
={}&\frac1{|\mathcal Z|}\sum_jw_j\sum_{\ell=2}^{|p_j|}
\log\frac{p_{\widetilde\theta}(p_{j,\ell}\mid p_{j,<\ell})}
{p_\theta(p_{j,\ell}\mid p_{j,<\ell})}\\
&+\frac{\sum_jw_j}{|\mathcal Z|}\sum_u\rho_u
\mathbb E_{q_u}\log\frac{P_{\widetilde\theta}(\tau\mid p_u)}{P_\theta(\tau\mid p_u)}.
\end{aligned}
\label{eq:hace_actual_loss_gain}
\end{equation}
The first term accounts for fitting prompt tokens. The second averages response gains over groups, with larger record weights for shorter concatenations. A lower total loss therefore does not establish an improved response likelihood for each instance.

The proposition uses decoded probabilities $\pi_\theta$, whereas the loss uses raw token probabilities $p_\theta$. If both decoded distributions are positive on the finite support of $q_u$, direct cancellation gives
\begin{equation}
\begin{aligned}
\mathbb E_{q_u}\log\frac{\pi_{\widetilde\theta}(\tau\mid p_u)}{\pi_\theta(\tau\mid p_u)}
={}&\mathbb E_{q_u}\log\frac{P_{\widetilde\theta}(\tau\mid p_u)}{P_\theta(\tau\mid p_u)}\\
&+\mathbb E_{q_u}\log
\frac{\pi_{\widetilde\theta}(\tau\mid p_u)P_\theta(\tau\mid p_u)}
{\pi_\theta(\tau\mid p_u)P_{\widetilde\theta}(\tau\mid p_u)}.
\end{aligned}
\label{eq:hace_decoding_gain}
\end{equation}
The left side is the decoded likelihood gain $g_t$ for that group. The final term is zero when decoding samples directly from the model and the recorded token product contains every generation probability factor, including any stopping token. Temperature scaling, probability truncation, or an omitted stopping factor can change it. If decoding gives zero probability to a selected response, the finite likelihood and divergence conditions do not hold.

Proposition~\ref{prop:hace_verification_gain} explains how agent fitting and later reference changes affect verification. The curriculum then selects specialist training instances, and pair acceptance determines their contribution to the exposure weights in Theorem~\ref{thm:hace_learning}. Increased verification alone does not establish that theorem's exposure or retained loss conditions.

\subsection{Task and Implementation Details}
\label{app:hace_tasks}
\label{app:hace_implementation}

Our experiments use candidate predictions in recommendation, protein function prediction, and material discovery. For an instance $i$, let $\mathcal C_i$ be the fixed reference candidate pool and $Y_i\subseteq\mathcal C_i$ the relevant candidates determined by the original label $y_i$. The harness exposes a view $\mathcal C_i^{S,H}\subseteq\mathcal C_i$. The definitions below apply to training and validation instances alike.

To evaluate an individual response, the parser extracts an ordered list of distinct candidate indices from the view. A response is valid if this list is nonempty. It is then completed with omitted candidates in specialist order. A response with no valid indices receives quality zero without completion. For any ordering $\rho$ of distinct candidates, let $|\rho|$ be its length and $\rho(r)$ the candidate at rank $r$. At a positive integer cutoff $k$, binary normalized discounted cumulative gain is~\citep{jarvelin2002cumulated}
\begin{equation}
\operatorname{NDCG}_k(\rho,Y_i)=
\begin{cases}
\displaystyle\frac{\sum_{r=1}^{\min(k,|\rho|)}\mathbf 1\{\rho(r)\in Y_i\}/\log_2(r+1)}
{\sum_{r=1}^{\min(k,|Y_i|)}1/\log_2(r+1)},&Y_i\neq\varnothing,\\[8pt]
0,&Y_i=\varnothing.
\end{cases}
\label{eq:hace_task_quality}
\end{equation}
For a valid response $\tau$, its quality $q_i^A(\tau)$ is this score on its completed ordering at the harness's verification cutoff $k_H$. The specialist quality $q_i^S$ is the same score on the specialist ordering of the same view. Verification requires $Y_i\neq\varnothing$ and $q_i^A(\tau)>q_i^S$. An empty view has no valid responses. Empty target sets are excluded from verification and calibration fitting, but receive zero in system evaluation.

The normalization uses all relevant candidates in the reference pool, so removing relevant candidates from a view does not lower the denominator. When $Y_i\neq\varnothing$,
\begin{equation}
q_i^A(\tau)\leq
\frac{\sum_{r=1}^{\min(k_H,|Y_i\cap\mathcal C_i^{S,H}|)}1/\log_2(r+1)}
{\sum_{r=1}^{\min(k_H,|Y_i|)}1/\log_2(r+1)}.
\label{eq:hace_candidate_ceiling}
\end{equation}
At most $|Y_i\cap\mathcal C_i^{S,H}|$ relevant candidates are available, and placing them first maximizes the discounted gain. This proves the bound and shows it is attainable by an ordering of the view. If the specialist attains this ceiling, further response sampling in that view cannot produce a positive contribution. Changing the harness can change the available view.

Joint selection orders each view by the combined score in Equation~\ref{eq:hace_fusion}. The agent ranking is the ordered candidate list returned by the evaluation response procedure, and specialist scores include any stored calibration. Each compared state fits $\lambda$ separately over the grid $\{0,0.25,0.5,1,2,4,8\}$. Let $\mathcal K$ be the task's finite nonempty set of positive integer evaluation cutoffs. Let $\mathcal D_{\mathrm{val}}$ and $\mathcal D_{\mathrm{test}}$ be the validation and test pools, both disjoint from $\mathcal D_{\mathrm{train}}$ and from each other. The comparison indices $\mathcal B_t$ refer to a batch from $\mathcal D_{\mathrm{val}}$. Split this batch into the same two nonempty folds for both compared states. On each fitting fold, choose the grid value maximizing mean NDCG averaged over $\mathcal K$, using only instances with nonempty target sets. Use zero if the fitting fold has no such instances. Evaluate on the other fold, then exchange their roles. All parsing, ordering, and tie rules remain fixed during the comparison. At zero weight, the joint ordering agrees with the specialist ordering within the same view.

Let $\rho_i$ be the joint ordering for evaluation instance $i$, formed with the weight fitted on the opposite fold. Then the score in Equation~\ref{eq:hace_acceptance} is
\begin{equation}
J_{\mathcal B_t}(S,H,A)=\frac1{|\mathcal B_t|}\sum_{i\in\mathcal B_t}
\frac1{|\mathcal K|}\sum_{k\in\mathcal K}\operatorname{NDCG}_k(\rho_i,Y_i).
\label{eq:hace_validation_score}
\end{equation}
The orderings depend on the evaluated $(S,H,A)$ and that state's fitted weights. This instance average weights the fold means by fold size and includes empty target sets with their zero scores. Pair selection compares joint predictions, while response verification compares an individual interpreted response with the specialist.


These evaluation rules determine the round procedure. Sample two archived harnesses and select the one with the higher score recorded at acceptance. Use a single entry directly, or the current harness if the archive is empty. Mutate one or two categorical fields to obtain $H'$. Under $(S_t,H',A_t)$, the verified curriculum prioritizes instances with at least one verified response, and the unsuccessful curriculum prioritizes those with none. Both select in pool order, truncate at the training budget, and fill shortfalls from remaining instances. The uncertainty curriculum prefers smaller differences between the two largest scores $s_i(c)+a_i(c)$. Uniform sampling bypasses exploration. The selected instances train $S'$ with Equation~\ref{eq:hace_model_loss}.

With $A_t$ fixed, evaluate the proposed and current pairs using Equation~\ref{eq:hace_validation_score}. Any specialist calibration uses a validation draw disjoint from this comparison batch. Accept a strictly positive improvement and record the harness and score, or restore the specialist parameters, optimizer, calibration, and harness. 

The retained pair then supplies the context for agent learning. Under $(\bar S_t,\bar H_t)$, collect fresh responses with $A_t$. Keep each eligible instance's original response with the largest positive contribution. The agent curriculum retains all such pairs or the half with the largest contributions, keeping at least one when available. The retained original tokens form $\mathcal Z_t$ for Equation~\ref{eq:hace_agent_loss}. An empty dataset leaves the agent unchanged. Set $S_{t+1}=\bar S_t$ and $H_{t+1}=\bar H_t$. The next round starts from $(S_{t+1},H_{t+1},A_{t+1})$. Labels enter specialist supervision, verification, and evaluation, and are excluded from prompts. Validation data are reused across rounds, and test data are reserved for final evaluation.

Figure~\ref{fig:overview} uses $\mathcal D_S$ for the labeled examples indexed by $\mathcal I_t^S$. An archive entry $(H_i,J_i)$ stores a harness and its score at acceptance, where $i$ indexes archive entries in that panel. The figure abbreviates $J_{\mathcal B_t}(S,H,A_t)$ as $J(S,H;A_t)$ and $\Delta J_t$ as $\Delta J$. Sampling $\tau\sim A_t$ means sampling from $\pi_{\theta_t}(\cdot\mid p_i)$ under the retained pair. Superscripts in $S^{(k)}$, $H^{(k)}$, and $A^{(k)}$ count component versions. The specialist and harness advance on acceptance, and the agent advances when it is updated.

\section{Additional Experiments}
\label{app:experiment_records}

\subsection{Evaluation Details}
\label{app:evaluation_metrics}

We fix disjoint training, validation, and test sets before coevolution. Specialist adaptation, agent learning, and training instance selection use only $\mathcal{D}_{\mathrm{train}}$. At each round, we compare the proposed and current specialist and harness pairs on the same batch sampled from $\mathcal{D}_{\mathrm{val}}$, holding the agent fixed. Both pairs use the same two folds, but fit their own $\lambda$ on one fold and are evaluated on the other. We then exchange the folds and average the resulting instance scores for joint selection. Internal specialist calibration uses a separate validation batch disjoint from the comparison batch. Validation data guide selection throughout coevolution, while $\mathcal{D}_{\mathrm{test}}$ is reserved for final evaluation and does not inform training, calibration, or selection.

For recommendation, we partition users across the three sets and train the initial specialist on histories excluding each user's last two interactions. During coevolution, we use the second to last interactions of training users for learning and those of validation users for calibration and selection. Final evaluation uses the last interaction of each test user. For protein function prediction, we follow the CAFA temporal split, restrict inputs to information available at $t_0$, and keep instances sharing a protein name in the same set. For material discovery, we partition by chemical system, keeping all candidate structures from the same system in one set. Below we describe the evaluation metrics.

\paragraph{Ranking metrics.}
Hit rate at cutoff $k$ (HR@$k$) is the fraction of recommendation queries whose held out target item appears among the first $k$ predictions. We report HR@10 in the main comparison. Normalized discounted cumulative gain (NDCG) assigns a binary relevance label to each candidate. For a ranking $\pi$, its discounted cumulative gain is
\begin{equation}
\mathrm{DCG}@k(\pi)=\sum_{r=1}^{\min(k,|\pi|)}\frac{\operatorname{rel}(\pi_r)}{\log_2(r+1)}.
\end{equation}
NDCG@$k$ divides this sum by the largest possible DCG@$k$ within the same fixed reference pool. It rewards both recovering relevant candidates and placing them earlier. The reported CAFA score averages NDCG@5 and NDCG@10, while the WBM score averages NDCG@3 and NDCG@5.

\paragraph{Protein function metrics.}
Precision is the fraction of predicted function terms that are correct, and recall is the fraction of relevant terms recovered. Recall@10 measures the fraction of relevant terms recovered among the first ten predictions. At prediction threshold $t$, their harmonic mean is $F_1(t)=2P(t)R(t)/(P(t)+R(t))$. The maximum score over thresholds is $F_{\max}=\max_t F_1(t)$. The area under the precision--recall curve (AUPR) summarizes the tradeoff between precision and recall across thresholds. Minimum semantic distance ($S_{\min}$) measures errors using the information added by each ontology term given its parents \citep{jiang2016cafa}. At each threshold, remaining uncertainty $\mathrm{ru}(t)$ is the average information in missed terms, and misinformation $\mathrm{mi}(t)$ is the average information in incorrectly predicted terms. The metric is
\begin{equation}
S_{\min}=\min_t\sqrt{\mathrm{ru}(t)^2+\mathrm{mi}(t)^2}.
\end{equation}
Lower values indicate less semantic error.

\paragraph{Materials metrics.}
For crystal stability prediction, $F_1$ is the harmonic mean of precision and recall for the stable class. The discovery acceleration factor (DAF) divides precision by the fraction of stable structures in the evaluation pool \citep{riebesell2025matbench}. A value of one matches the expected precision of random selection. Mean absolute error (MAE) averages the absolute differences between predicted and reference target values. The coefficient of determination ($R^2$) equals one minus the ratio of squared prediction error to the squared error of predicting the mean reference value. It equals one for perfect predictions, zero for this constant baseline, and is negative when predictions have greater squared error. Higher values are better for all metrics except $S_{\min}$ and MAE.

\subsection{Complete Scientific Results}
\begin{table}[!htb]
\centering
\caption{Protein function prediction on CAFA with two frozen backbones.}
\label{tab:full_protein}
\fontsize{9}{10.6}\selectfont
\setlength{\tabcolsep}{4pt}
\renewcommand{\arraystretch}{1.14}
\begin{tabularx}{\textwidth}{l*{5}{>{\centering\arraybackslash}X}}
\toprule
Method & $F_{\max}\uparrow$ & AUPR$\uparrow$ & $S_{\min}\downarrow$ & NDCG$\uparrow$ & Recall@10$\uparrow$ \\
\midrule
\multicolumn{6}{l}{\textbf{ESM-2}} \\

Base & $0.2218_{\color{black!50}\pm 0.0014}$ & $0.1078_{\color{black!50}\pm 0.0012}$ & $12.557_{\color{black!50}\pm 0.060}$ & $0.6285_{\color{black!50}\pm 0.0019}$ & $0.7179_{\color{black!50}\pm 0.0014}$ \\
ReST-EM & $0.2219_{\color{black!50}\pm 0.0014}$ & $0.1079_{\color{black!50}\pm 0.0012}$ & $12.557_{\color{black!50}\pm 0.060}$ & $0.6286_{\color{black!50}\pm 0.0018}$ & $0.7180_{\color{black!50}\pm 0.0012}$ \\
Meta-Harness & $0.2224_{\color{black!50}\pm 0.0013}$ & $0.1076_{\color{black!50}\pm 0.0012}$ & $12.553_{\color{black!50}\pm 0.059}$ & $0.6309_{\color{black!50}\pm 0.0015}$ & $0.7245_{\color{black!50}\pm 0.0007}$ \\
OpenEvolve & $0.2230_{\color{black!50}\pm 0.0012}$ & $0.1087_{\color{black!50}\pm 0.0011}$ & $12.549_{\color{black!50}\pm 0.060}$ & $0.6320_{\color{black!50}\pm 0.0018}$ & $0.7222_{\color{black!50}\pm 0.0016}$ \\
DGM & $0.2219_{\color{black!50}\pm 0.0014}$ & $0.1075_{\color{black!50}\pm 0.0012}$ & $12.557_{\color{black!50}\pm 0.060}$ & $0.6325_{\color{black!50}\pm 0.0019}$ & $0.7261_{\color{black!50}\pm 0.0016}$ \\
TTT-Discover & $0.2219_{\color{black!50}\pm 0.0014}$ & $0.1078_{\color{black!50}\pm 0.0013}$ & $12.557_{\color{black!50}\pm 0.060}$ & $0.6287_{\color{black!50}\pm 0.0018}$ & $0.7185_{\color{black!50}\pm 0.0013}$ \\

\textbf{HACE} & $\mathbf{0.2438}_{\color{black!50}\pm 0.0019}$ & $\mathbf{0.1258}_{\color{black!50}\pm 0.0015}$ & $\mathbf{12.494}_{\color{black!50}\pm 0.048}$ & $\mathbf{0.6745}_{\color{black!50}\pm 0.0016}$ & $\mathbf{0.7615}_{\color{black!50}\pm 0.0015}$ \\
\midrule
\multicolumn{6}{l}{\textbf{ProtT5-XL}} \\

Base & $0.2221_{\color{black!50}\pm 0.0005}$ & $0.1078_{\color{black!50}\pm 0.0002}$ & $13.127_{\color{black!50}\pm 0.074}$ & $0.6241_{\color{black!50}\pm 0.0011}$ & $0.7122_{\color{black!50}\pm 0.0015}$ \\
ReST-EM & $0.2221_{\color{black!50}\pm 0.0005}$ & $0.1079_{\color{black!50}\pm 0.0003}$ & $13.127_{\color{black!50}\pm 0.074}$ & $0.6242_{\color{black!50}\pm 0.0011}$ & $0.7131_{\color{black!50}\pm 0.0016}$ \\
Meta-Harness & $0.2229_{\color{black!50}\pm 0.0006}$ & $0.1085_{\color{black!50}\pm 0.0002}$ & $13.122_{\color{black!50}\pm 0.075}$ & $0.6261_{\color{black!50}\pm 0.0017}$ & $0.7156_{\color{black!50}\pm 0.0020}$ \\
OpenEvolve & $0.2230_{\color{black!50}\pm 0.0003}$ & $0.1082_{\color{black!50}\pm 0.0003}$ & $13.121_{\color{black!50}\pm 0.072}$ & $0.6254_{\color{black!50}\pm 0.0014}$ & $0.7139_{\color{black!50}\pm 0.0015}$ \\
DGM & $0.2235_{\color{black!50}\pm 0.0006}$ & $0.1088_{\color{black!50}\pm 0.0001}$ & $\mathbf{13.118}_{\color{black!50}\pm 0.075}$ & $0.6273_{\color{black!50}\pm 0.0016}$ & $0.7175_{\color{black!50}\pm 0.0016}$ \\
TTT-Discover & $0.2221_{\color{black!50}\pm 0.0005}$ & $0.1079_{\color{black!50}\pm 0.0003}$ & $13.127_{\color{black!50}\pm 0.074}$ & $0.6241_{\color{black!50}\pm 0.0011}$ & $0.7135_{\color{black!50}\pm 0.0016}$ \\

\textbf{HACE} & $\mathbf{0.2317}_{\color{black!50}\pm 0.0005}$ & $\mathbf{0.1177}_{\color{black!50}\pm 0.0010}$ & $13.127_{\color{black!50}\pm 0.072}$ & $\mathbf{0.6739}_{\color{black!50}\pm 0.0013}$ & $\mathbf{0.7553}_{\color{black!50}\pm 0.0012}$ \\
\bottomrule
\end{tabularx}
\end{table}

\begin{table}[!htb]
\centering
\caption{Material discovery task on WBM with CHGNet. }
\label{tab:full_materials}
\fontsize{9}{10.6}\selectfont
\setlength{\tabcolsep}{4pt}
\renewcommand{\arraystretch}{1.14}
\begin{tabularx}{\textwidth}{l*{5}{>{\centering\arraybackslash}X}}
\toprule
Method & NDCG$\uparrow$ & $F_1\uparrow$ & DAF$\uparrow$ &
\shortstack{MAE$\downarrow$} & $R^2\uparrow$ \\
\midrule
Base
& $0.6388_{\color{black!50}\pm 0.0011}$
& $0.5159_{\color{black!50}\pm 0.0003}$
& $1.644_{\color{black!50}\pm 0.003}$
& $0.1049_{\color{black!50}\pm 0.0002}$
& $-1.289_{\color{black!50}\pm 0.001}$ \\

ReST-EM
& $0.6496_{\color{black!50}\pm 0.0006}$
& $0.5508_{\color{black!50}\pm 0.0004}$
& $2.172_{\color{black!50}\pm 0.005}$
& $0.0756_{\color{black!50}\pm 0.0002}$
& $-0.323_{\color{black!50}\pm 0.001}$ \\

Meta-Harness
& $0.6513_{\color{black!50}\pm 0.0008}$
& $0.5510_{\color{black!50}\pm 0.0004}$
& $2.171_{\color{black!50}\pm 0.003}$
& $0.0754_{\color{black!50}\pm 0.0003}$
& $-0.281_{\color{black!50}\pm 0.001}$ \\

OpenEvolve
& $0.6486_{\color{black!50}\pm 0.0008}$
& $0.5463_{\color{black!50}\pm 0.0003}$
& $2.193_{\color{black!50}\pm 0.004}$
& $0.0755_{\color{black!50}\pm 0.0002}$
& $-0.264_{\color{black!50}\pm 0.001}$ \\

DGM
& $0.6496_{\color{black!50}\pm 0.0007}$
& $0.5508_{\color{black!50}\pm 0.0003}$
& $2.183_{\color{black!50}\pm 0.003}$
& $0.0748_{\color{black!50}\pm 0.0003}$
& $-0.251_{\color{black!50}\pm 0.001}$ \\

TTT-Discover
& $0.6474_{\color{black!50}\pm 0.0007}$
& $0.5504_{\color{black!50}\pm 0.0003}$
& $2.203_{\color{black!50}\pm 0.003}$
& $0.0748_{\color{black!50}\pm 0.0002}$
& $-0.284_{\color{black!50}\pm 0.001}$ \\

\textbf{HACE}
& $\mathbf{0.6541}_{\color{black!50}\pm 0.0010}$
& $\mathbf{0.5513}_{\color{black!50}\pm 0.0004}$
& $\mathbf{2.242}_{\color{black!50}\pm 0.004}$
& $\mathbf{0.0729}_{\color{black!50}\pm 0.0003}$
& $\mathbf{-0.203}_{\color{black!50}\pm 0.001}$ \\
\bottomrule
\end{tabularx}
\end{table}

\subsection{More Experimental Details}
\label{app:experimental_details}

\paragraph{Agent and model scaling.}
Table~\ref{tab:model_scaling} varies specialist size and agent size in separate sweeps. ESM-2 backbone representations remain frozen while the prediction head is trained. The agent variants share a 30-round configuration with ESM-2 650M. In this sweep, Base is the initial specialist score, shared across the agent variants, and $\Delta$ is the reported mean within-run gain. The specialist sweep uses a separate configuration. Subscripts report the standard deviation of the gains. The HACE means for ESM-2 35M, 150M, and 3B are obtained by adding the reported Base means and gains. Both the specialist sweep and all agent variants use three random seeds.

\paragraph{Convergence on base.} We begin with specialists already trained to convergence under their original schedules. Continued supervised training alone does not by itself indicate which further updates help, and our ablations show that the acceptance gate in Equation~\ref{eq:hace_acceptance} is what makes continued training productive. This setting tests whether interaction with an evolving agent can create additional opportunities beyond that plateau. We define convergence as follows: continuing the original supervised schedule for an additional $3$ epochs changes the reported metric by less than $0.001$ in every domain.

\paragraph{Evolution records.}
The lower panels of Figure~\ref{fig:case_study} show selection scores from one recorded run in each domain. Grey points are the available per-round scores, and the lines track their running maxima. The recommendation trajectory covers 30 rounds, while the protein and materials trajectories cover 50 rounds. The CAFA record contains scores for 18 rounds. Missing entries are omitted. These trajectories describe the scores observed during selection, while the main tables report aggregate evaluation results. The recommendation case was collected in a separate logging run with the same configuration, seed, and data as its trajectory. The upper panels compare specialist and combined predictions on the same instance. The protein example is recorded at round 40 and the materials example at round 0.

\paragraph{Hyperparameters.}
We sample 192 training instances per round. For each instance, the agent generates $m$ responses with temperature 1.0 and top-$p$ 0.95, where $m=8$ for recommendation and crystal stability and $m=4$ for protein function. The generation batch size is 96. Prompts exceeding 2560 tokens are truncated by retaining the first 853 and the last 1707 tokens. For joint selection, we sample $|\mathcal{B}_t|=\min(256,|\mathcal{D}_{\mathrm{val}}|)$ validation instances. We train the agent with LoRA adapters ($r=16$, $\alpha=32$) using AdamW with a learning rate of $10^{-5}$, gradient clipping at 1.0, and a micro-batch size of 2 sequences. Specialist adaptation in Equation~\ref{eq:hace_model_loss} uses 200 steps for recommendation and crystal stability and 750 steps for protein function, all at a learning rate of $10^{-4}$.

\paragraph{Hardware.} All experiments are conducted on NVIDIA B200 GPU clusters, where each node contains 8 GPUs. 

\begin{table}[!htb]
\centering
\caption{HACE is agnostic to specialist architectures and adaptation on different learning components. }
\label{tab:studied_models}
\begingroup
\fontsize{8.8}{10.3}\selectfont
\setlength{\tabcolsep}{4pt}
\renewcommand{\arraystretch}{1.13}
\begin{tabularx}{\linewidth}{@{}l l >{\raggedright\arraybackslash}X l@{}}
\toprule
Domain & Model & Architecture & Updated parameters \\
\midrule
Recommendation & SASRec & Causal self-attention Transformer & Model parameters \\
& HSTU & Causal sequential transducer with pointwise aggregated attention & Model parameters \\
& Wukong & Stacked factorization-machine network & Model parameters \\
\midrule
Protein & ESM-2 & Encoder-only protein Transformer & Prediction head \\
& ProtT5-XL & Protein T5 with encoder-decoder pretraining & Prediction head \\
\midrule
Materials & CHGNet & Charge-informed graph neural potential & Model parameters \\
\bottomrule
\end{tabularx}
\endgroup
\end{table}

\begin{table}[!htb]
\centering
\caption{How exploration guides the two learners. Relative improvement determines eligibility for different learning objectives.}
\label{tab:reciprocal_learning_outcomes}
{
\small
\setlength{\tabcolsep}{4pt}
\renewcommand{\arraystretch}{1.12}
\begin{tabularx}{\linewidth}{@{}>{\raggedright\arraybackslash}X>{\raggedright\arraybackslash}X>{\raggedright\arraybackslash}X@{}}
\toprule
Exploration outcome & Agent learning & Specialist learning \\
\midrule
At least one sampled ranking improves the specialist & Best positive response is eligible for the agent curriculum & Instance is eligible for the verified curriculum \\

No sampled ranking improves the specialist & No response from this instance enters the agent update & Instance is eligible for the unsuccessful curriculum \\
\bottomrule
\end{tabularx}
\par\vspace{3pt}
\raggedright\footnotesize
Specialist selection uses exploration before joint selection. Agent learning uses fresh responses afterward. The specialist always trains on task labels. Uniform and uncertainty curricula, as well as budget completion, may select instances from either row.\par
}
\end{table}

\FloatBarrier
\section{Harness Engineering}
\label{app:harness}

\newtcolorbox{promptbox}[1]{
  title={#1},
  colback=white,
  colframe=black!25,
  colbacktitle=black!65,
  coltitle=white,
  colupper=black,
  fonttitle=\small\bfseries,
  fontupper=\small,
  boxrule=0.4pt,
  titlerule=0pt,
  arc=1mm,
  boxsep=0pt,
  left=3mm, right=3mm,
  top=2.5mm, bottom=2.5mm,
  toptitle=2mm, bottomtitle=2mm,
  before upper={
    \setlength{\parindent}{0pt}
    \setlength{\parskip}{4pt}
  },
  before skip=8pt,
  after skip=8pt
}

\subsection{Prompts for Recommendation}
\label{app:harness_rec}

We use a common recommendation interface for MovieLens 1M and the
Amazon Beauty, CDs, Sports, and Toys datasets. For each user, it
combines an interaction history with a candidate ranking from SASRec,
HSTU, or Wukong. Harness search controls the history length, available
item metadata, and the partition between leading predictions and
candidate alternatives. The prompt uses the same item terminology
across all five datasets.

\begin{promptbox}{Prompt: Recommendation Systems}
Use the user's interaction history and the available item metadata
to select the \{k\} candidate alternatives that best match the user's
preferences. The specialist's leading predictions provide additional
context.

\textbf{Interaction history}\par
\{recent interactions in sequence, with item titles and available metadata\}

\textbf{Specialist's leading predictions}\par
\{leading items in specialist order, with available metadata\}

\textbf{Candidate alternatives}\par
1. \{item title and available metadata\}\\
2. \{item title and available metadata\}\\
\ldots

\textbf{Output}\par
Return exactly \{k\} distinct IDs from Candidate alternatives,
ordered from best to worst and separated by spaces.
Return no other text.
\end{promptbox}

When $k=1$ and the leading block is nonempty, the following instruction
uses the same inputs and allows the agent to retain the specialist's
selection. Otherwise, the ranking prompt above applies.

\begin{promptbox}{Prompt: Recommendation with a Single Selection}
Use the user's interaction history and the available item metadata
to identify the best candidate alternative. Compare it with the
lowest ranked item among the specialist's leading predictions.

\textbf{Output}\par
If this alternative is a better match for the user, return its
integer ID. Otherwise, return 0. Return one integer and no other text.
\end{promptbox}

\begin{table}[!htb]
\centering
\caption{Recommendation harness fields. ``Fixed'' denotes a value
outside the search space.}
\label{tab:harness-reco}
\small
\setlength{\tabcolsep}{4pt}
\renewcommand{\arraystretch}{1.10}
\begin{tabularx}{\linewidth}{@{}l l l >{\raggedright\arraybackslash}X@{}}
\toprule
Field & Initial & Candidate values & Role \\
\midrule
\texttt{n\_candidates} & 50 & Fixed & Specialist shortlist size \\
\texttt{n\_history} & 15 & 5, 10, 15, 25, 40 & Number of history entries \\
\texttt{top\_k} & 5 & 1, 3, 5, 10 & Requested ranking length \\
\texttt{promote\_from} & 10 & 0, 5, 10, 20 & Leading candidates shown separately \\
\texttt{n\_promote} & 1 & Fixed & Promotion count when $k=1$ \\
\texttt{show\_year} & True & True, False & Release year metadata \\
\texttt{show\_genre} & True & True, False & Category metadata \\
\bottomrule
\end{tabularx}
\end{table}

\FloatBarrier
\Needspace{0.72\textheight}
\subsection{Prompts for Protein Function Prediction}
\label{app:harness_protein}

For protein function prediction, the candidates are Gene Ontology
terms. The prompt combines protein metadata and available annotations
with the specialist's ranking. Harness search controls the number of
annotations, the length of term definitions, and the metadata shown
to the agent. The agent returns candidate IDs using the same ranking
format as in recommendation.

\begin{promptbox}{Prompt: Protein Function Prediction}
Use the protein information, available annotations, and Gene Ontology
term descriptions to select the \{k\} candidate alternatives most
likely to apply to this protein. The specialist's leading predictions
provide additional context.

\textbf{Protein information}\par
Protein: \{protein name\}\\
Organism: \{organism, when enabled\}\\
Target ontology aspect: \{aspect, when enabled\}\\
Available annotations: \{known terms\}

\textbf{Specialist's leading predictions}\par
\{GO terms in specialist order, with available definitions\}

\textbf{Candidate alternatives}\par
1. \{GO term name and available definition\}\\
2. \{GO term name and available definition\}\\
\ldots

\textbf{Output}\par
Return exactly \{k\} distinct IDs from Candidate alternatives,
ordered from most to least likely and separated by spaces.
Return no other text.
\end{promptbox}

\begin{table}[!htb]
\centering
\caption{Presentation fields of the protein harness. Experience
selection is described in Appendix~\ref{app:hace_implementation}.}
\label{tab:harness-prot}
\small
\setlength{\tabcolsep}{4pt}
\renewcommand{\arraystretch}{1.10}
\begin{tabularx}{\linewidth}{@{}l l l >{\raggedright\arraybackslash}X@{}}
\toprule
Field & Initial & Candidate values & Role \\
\midrule
\texttt{n\_candidates} & 50 & Fixed & Specialist shortlist size \\
\texttt{n\_known} & 10 & 0, 5, 10, 20 & Number of known annotations \\
\texttt{top\_k} & 5 & 1, 3, 5, 10 & Requested ranking length \\
\texttt{promote\_from} & 10 & 0, 5, 10, 20 & Leading candidates shown separately \\
\texttt{def\_chars} & 120 & 0, 120, 300 & Definition length, zero for name only \\
\texttt{show\_aspect} & True & True, False & Target ontology aspect \\
\texttt{show\_organism} & True & True, False & Source organism \\
\bottomrule
\end{tabularx}
\end{table}

\Needspace{0.78\textheight}
\subsection{Prompts for Material Discovery}
\label{app:harness_materials}

For material discovery, the agent compares candidate crystals within
a chemical system. The prompt provides compositions and available
structure prototypes, together with the specialist's leading
predictions. Harness search controls the prototype length and whether
element information and specialist ranks are shown. Candidate IDs
identify the agent's selections independently of the specialist ranks.

\begin{promptbox}{Prompt: Crystal Stability Prediction}
Use the chemical system, candidate compositions, and available
structural descriptors to select the \{k\} candidate alternatives
most likely to be thermodynamically stable. The specialist's leading
predictions provide additional context.

\textbf{Chemical system}\par
System: \{chemical system\}\\
Elements: \{element set, when enabled\}

\textbf{Specialist's leading predictions}\par
\{formulas in specialist order, with available structure prototypes\}

\textbf{Candidate alternatives}\par
1. \{formula, structure prototype, and specialist rank when enabled\}\\
2. \{formula, structure prototype, and specialist rank when enabled\}\\
\ldots

\textbf{Output}\par
Return exactly \{k\} distinct IDs from Candidate alternatives,
ordered from most to least likely and separated by spaces.
Use the IDs at the start of the candidate rows. Return no other text.
\end{promptbox}

\begin{table}[!htb]
\centering
\caption{Presentation fields of the crystal stability harness.
Prototype lengths are measured in characters.}
\label{tab:harness-crys}
\small
\setlength{\tabcolsep}{4pt}
\renewcommand{\arraystretch}{1.10}
\begin{tabularx}{\linewidth}{@{}l l l >{\raggedright\arraybackslash}X@{}}
\toprule
Field & Initial & Candidate values & Role \\
\midrule
\texttt{n\_candidates} & 20 & Fixed & Specialist shortlist size \\
\texttt{top\_k} & 3 & 1, 2, 3, 5 & Requested ranking length \\
\texttt{promote\_from} & 2 & 0, 1, 2, 3 & Leading candidates shown separately \\
\texttt{proto\_chars} & 24 & 0, 24, 48 & Structure prototype length \\
\texttt{show\_elements} & True & True, False & Element set \\
\texttt{show\_rank} & True & True, False & Specialist rank per candidate \\
\bottomrule
\end{tabularx}
\end{table}
\end{document}